\newif\ifarxiv\arxivtrue

\newif\ifnotarxiv
\ifarxiv \notarxivfalse \else \notarxivtrue \fi

\ifarxiv
    \newcommand{\arxiv}[1]{#1}
\else
    \newcommand{\arxiv}[1]{}
\fi

\ifarxiv
    \newcommand{\notarxiv}[1]{}
\else
    \newcommand{\notarxiv}[1]{#1}
\fi

\ifarxiv
    \documentclass[11pt]{article}
\else
    \documentclass[letterpaper]{article} 
    \usepackage{aaai23}  
    \usepackage{times}  
    \usepackage{helvet}  
    \usepackage{courier}  
    \usepackage[hyphens]{url}  
    \usepackage{graphicx} 
    \usepackage{natbib}  
    \usepackage{caption} 
\fi
\arxiv{\usepackage[authoryear]{natbib}}
\arxiv{\usepackage[margin=1in]{geometry}}

\ifarxiv
    \usepackage[dvipsnames]{xcolor}
    \usepackage{color}
    \definecolor{cornellred}{rgb}{0.7, 0.11, 0.11}
    \definecolor{dgreen}{rgb}{0.0, 0.5, 0.0}
    \definecolor{ballblue}{rgb}{0.13, 0.67, 0.8}
    \definecolor{royalblue(web)}{rgb}{0.25, 0.41, 0.88}
    \definecolor{bleudefrance}{rgb}{0.19, 0.55, 0.91}
    \definecolor{royalazure}{rgb}{0.0, 0.22, 0.66}
    \usepackage[breaklinks]{hyperref}
    \hypersetup{
        colorlinks = true,
        linkcolor = cornellred,
        citecolor = royalazure,
        linkbordercolor = white,
        urlcolor  = cornellred
    }
    \usepackage{float}
\fi

\ifarxiv
    \usepackage{microtype}
    \usepackage{graphicx}
    \usepackage{subcaption}
    \usepackage{booktabs} 
    \usepackage[ruled,vlined,linesnumbered]{algorithm2e}
\fi

\ifnotarxiv
    \usepackage{algorithm}
    \usepackage{algorithmic}
\fi

\usepackage{newfloat}
\usepackage{listings}
\ifnotarxiv
    \DeclareCaptionStyle{ruled}{labelfont=normalfont,labelsep=colon,strut=off} 
    \floatstyle{ruled}
    \newfloat{listing}{tb}{lst}{}
    \floatname{listing}{Listing}
\fi

\usepackage{verbatim}
\usepackage{subcaption} 
\usepackage{multirow} 
\usepackage{makecell} 

\usepackage{paralist}

\usepackage{tikz}
\usetikzlibrary{arrows.meta}
\usetikzlibrary{shapes.multipart}
\usepackage{forest}
\usepackage{mathtools}
\usepackage[framemethod=tikz]{mdframed}
\usepackage{pgfplots}

\title{Tree Learning: Optimal Sample Complexity and Algorithms}
\ifarxiv
    \author{
       Dmitrii Avdiukhin\\ Indiana University\\ davdyukh@iu.edu \and
       Grigory Yaroslavtsev\\ George Mason University\\ grigory@grigory.us \and
       Danny Vainstein\\ Tel-Aviv University\\ dannyvainstein@gmail.com \and
       Orr Fischer\\ Weizmann Institute\\ of Science\\ orr.fischer@weizmann.ac.il \and
       Sauman Das\\ Thomas Jefferson High School\\ for Science and Technology\\ 2023sdas@tjhsst.edu \and
       Faraz Mirza\\ Thomas Jefferson High School\\ for Science and Technology\\ 2023fmirza@tjhsst.edu
    }
\else
    \author{
       Dmitrii Avdiukhin$^\ddag$\textsuperscript{\rm 1},
       Grigory Yaroslavtsev$^\ddag$\textsuperscript{\rm 2},
       Danny Vainstein$^\ddag$\textsuperscript{\rm 3},
       Orr Fischer$^\ddag$\textsuperscript{\rm 4},
       Sauman Das\textsuperscript{\rm 5},
       Faraz Mirza\textsuperscript{\rm 5}
    }
    \affiliations{
       \textsuperscript{\rm 1} Indiana University, Department of Computer Science\\
       \textsuperscript{\rm 2} George Mason University, Department of Computer Science \footnote{Supported by NSF CCF-1657477 and Facebook Faculty Award.}\\
       \textsuperscript{\rm 3} Tel-Aviv University, Blavatnik School of Computer Science \\
       \textsuperscript{\rm 4} Weizmann Institute of Science, Department of Computer Science and Applied Mathematics \footnote{This project is partially funded by the European Research Council (ERC) under the European Union’s Horizon 2020 research and innovation programme (grant agreement No. 949083).}\\
       \textsuperscript{\rm 5} Thomas Jefferson High School for Science and Technology\\
       davdyukh@iu.edu, orr.fischer@weizmann.ac.il, dannyvainstein@gmail.com, grigory@grigory.us, 2023sdas@tjhsst.edu,
       2023fmirza@tjhsst.edu
    }
    \fi

\usepackage{amsmath}
\usepackage{amsthm}
\usepackage{amssymb}
\usepackage{nicefrac}
\usepackage{tabularx}
\usetikzlibrary{shapes.geometric}

\newenvironment{example}[1][Example.]{\begin{trivlist}
\item[\hskip \labelsep {\bfseries #1}]}{\end{trivlist}}

\renewcommand{\include}{\input}

\newcommand{\eps}{\varepsilon}

\newcommand{\genf}{e}

\newcommand{\bbr}[1]{\big\{#1\big\}}

\usepackage{scalerel,stackengine}
\newcommand\pig[1]{\scalerel*[5.5pt]{\Big#1}{%
  \ensurestackMath{\addstackgap[1.5pt]{\big#1}}}}
\newcommand\pigl[1]{\mathopen{\pig{#1}}}
\newcommand\pigr[1]{\mathclose{\pig{#1}}}
\newcommand{\card}[1]{\pigl| #1 \pigr|}

\DeclareMathOperator{\ldim}{LDim}
\DeclareMathOperator{\ndim}{NDim}

\DeclareMathOperator{\lca}{LCA}
\DeclareMathOperator{\argmin}{argmin}
\DeclareMathOperator{\err}{err}
\DeclareMathOperator{\Rank}{Rank}
\DeclareMathOperator{\Ind}{Ind}

\DeclareMathOperator{\polylog}{\mathrm{polylog}}

\newcommand{\C}[3]{[#1, #2 \, |\, #3]}
\newcommand{\TW}[3]{[#1 | #2 | #3]}

\newcommand{\orient}[1]{\protect\overrightarrow{#1}}

\usepackage[disable]{todonotes}
\usepackage{xcolor}

\tikzset{
  treenode/.style = {align=center, inner sep=0pt, text centered, font=\sffamily},
  my triangle/.style={-{Triangle[width=\the\dimexpr1.8\pgflinewidth,length=\the\dimexpr1.5\pgflinewidth]}},
}

\ifarxiv
    \newcommand{\defthm}[1]{\newmdtheoremenv[roundcorner=10, innerleftmargin=7, innerrightmargin=7, innertopmargin=-4pt, leftmargin=-7, rightmargin=-7, backgroundcolor=#1!1.5, nobreak=true]}
\else
    \newcommand{\defthm}[1]{\newtheorem}
\fi

\defthm{blue}{theorem}{Theorem}[section]
\newtheorem{corollary}[theorem]{Corollary}
\defthm{yellow}{lemma}[theorem]{Lemma}

\defthm{green}{definition}[theorem]{Definition}
\defthm{red}{assumption}{Assumption}

\forestset{
  solid node/.style={circle,draw,inner sep=1.2,fill=black},
}

\forestset{
  tria/.style={
    isosceles triangle,
    isosceles triangle apex angle=60,
    draw,
    shape border rotate=90,
    minimum size=.7cm,
    child anchor=apex,
    anchor=apex,
    fill=none,
    scale=0.8,
  },
  triangle/.style={
    isosceles triangle, isosceles triangle apex angle=45, child anchor=parent, shape border uses incircle, shape border rotate=90, draw,
    delay={
      if={
        >O_> {!u.n children}{1}
      }{
        replace by/.process={ Ow {id} {[, coordinate, tier=##1, append, l'=0pt]}},
        for siblings={tier/.option=id},
        no edge,
      }{},
    },
  },
}%

\newenvironment{tikzext}
{}{}

\begin{document}


\maketitle

\begin{abstract}
We study the problem of learning a hierarchical tree representation of data from labeled samples, taken from an arbitrary (and possibly adversarial) distribution. Consider a collection of data tuples labeled according to their hierarchical structure. The smallest number of such tuples required in order to be able to accurately label subsequent tuples is of interest for data collection in machine learning. We present optimal sample complexity bounds for this problem in several learning settings, including (agnostic) PAC learning and online learning. Our results are based on tight bounds of the Natarajan and Littlestone dimensions of the associated problem. The corresponding tree classifiers can be constructed efficiently in near-linear time.
\end{abstract}
\ifnotarxiv
    \def\thefootnote{$\ddag$}\footnotetext{These authors contributed equally to this work}
\fi
\newcommand{\cD}{\mathcal D}

\section{Introduction}
The algorithmic problem of constructing hierarchical data representations has been of major importance for many decades, due to its applications in statistics~\citep{W63,GR69}, entomology~\citep{MS57}, plant biology~\citep{S48}, genomics~\citep{ESBB98} and other domains.
Efficient collection of labeled data is a problem of key importance for construction of hierarchical data representations. In this paper we consider the problem of constructing tree representations of data from labeled samples, focusing on understanding the optimal number of samples required for this task. 

The most basic type of a label that allows one to construct a tree representation consists of a triplet of points $(x, y, z)$ labeled according to the induced hierarchical structure within the triplet.
For example, given images of a cat, a dog, and a plane, the label would describe the cat and the dog as being more similar to each other than to the plane.
This ``odd one out'' type of label is the simplest one to collect in a crowdsourcing setting.
In this paper, we focus on understanding the number of such labeled samples required in order to construct a tree representation of the underlying data, which enables one to accurately predict subsequent labels in the future. This is then further generalized to include larger labeled subsets of data. Examples of possible trees consistent with certain triplet and tuple labelings are shown in Figure~\ref{fig:consistency}.


\ifarxiv
    \newcommand{\figwidth}{0.48\textwidth}
\else
    \newcommand{\figwidth}{0.22\textwidth}
\fi

\begin{figure}[t!]
    \centering
    \begin{subfigure}[b]{\figwidth}
        \centering
        \newcommand{\treeparams}{for tree={s sep=0.55cm,l-=1.5em,solid node}}
\bracketset{action character=@}

\begin{tikzext}
    \begin{forest}
        [ ,@\treeparams,
            [
              [ ,label={below:$a$} ]
              [ ,label={below:$b$} ]
            ]
            [ ,label={below:$c$}]
        ];
    \end{forest}
    \begin{forest}
        [ ,@\treeparams,
            [
              [ ,label={below:$b$} ]
              [ ,label={below:$d$} ]
            ]
            [ ,label={below:$e$}]
        ];
    \end{forest}
    \begin{tikzpicture}
        \centering
        \draw[line width=10pt,my triangle,
            postaction={draw,white,line width=8pt,my triangle,shorten >=2pt,shorten <=1pt}](0,0) -- (1,0);
        \node at (0,-1) {};
    \end{tikzpicture}
    \begin{forest}
        [ ,@\treeparams,
          [
            [
              [ ,label={below:$a$} ]
              [ ,label={below:$d$} ]
            ]
            [ ,label={below:$b$}]
          ]
          [
              [ ,label={below:$c$} ]      
              [ ,label={below:$e$}]
          ]
        ];
    \end{forest}
\end{tikzext}
        \caption{Triplet labelings and a tree consistent with the labelings.}
        \label{fig:constraint_sat_1}
    \end{subfigure}
    \ifarxiv \hspace{1em} \else \hfill \fi
    \begin{subfigure}[b]{\figwidth}
        \centering
        \newcommand{\treeparams}{for tree={s sep=0.4cm,l-=1.5em,solid node}}
\bracketset{action character=@}

\begin{tikzext}
    \begin{forest}
        [ ,@\treeparams,
          [
            [
              [ ,label={below:$a$} ]
              [ ,label={below:$b$} ]
            ]
            [ ,label={below:$c$}]
          ]
          [ ,label={below:$d$}]
        ];
    \end{forest}
    \begin{forest}
        [ ,@\treeparams,
          [
            [ ,label={below:$a$} ]
            [ ,label={below:$b$} ]
          ]
          [
            [ ,label={below:$d$}]
            [ ,label={below:$e$}]
          ]
        ];
    \end{forest}
    \begin{tikzpicture}
        \centering
        \draw[line width=10pt,my triangle,
            postaction={draw,white,line width=8pt,my triangle,shorten >=2pt,shorten <=1pt}](0,0) -- (1,0);
        \node at (0,-1.2) {};
    \end{tikzpicture}
    \begin{forest}
        [ ,@\treeparams,
          [
            [
              [ ,label={below:$a$} ]
              [ ,label={below:$b$} ]
            ]
            [ ,label={below:$c$}]
          ]
          [
              [ ,label={below:$d$} ]      
              [ ,label={below:$e$}]
          ]
        ];
    \end{forest}
\end{tikzext}
        \caption{$4$-tuple labelings and a tree consistent with the labelings.}
        \label{fig:constraint_sat_2}
    \end{subfigure}
    \hfill
    \caption{An example for the triplet and $k$-tuple setting, where two labeled queries and a tree consistent with the queries are shown.
    The tree satisfies all hierarchical relations of both labeled queries.
    For example, in Figure~\ref{fig:constraint_sat_1}, the lowest common ancestor of $a$ and $b$ (denoted $\lca(a, b)$) is below  $\lca(a,c)$ and $\lca(b, c)$, as in its first labeled query.}
\label{fig:consistency}
\end{figure}
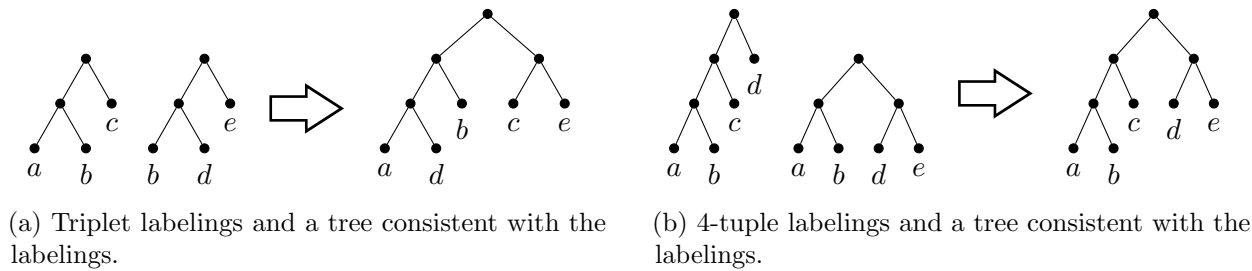

\subsection{Our Results}
Let $n$ be the number of points in the dataset.
We present results in two settings: PAC learning and online learning. In both cases, our objective is to build a classifier that given access to labeled tuples can predict labels for the subsequent tuples with high probability. 

\paragraph{PAC learning} In the \emph{Probably Approximately Correct (PAC) learning} setting~\citep{V84}, the tuples are generated from a fixed unknown distribution $\cD$.
The goal is to predict their labels with probability at least $1 - \epsilon$, i.e. achieve the error rate at most $\epsilon$, while providing this guarantee with overall probability at least $1 - \delta$.
We distinguish between the two cases: in the \emph{realizable} setting it is assumed that the labels are consistent (i.e. there exists a tree that respects all observed labels), while in the \emph{agnostic} PAC-learning setting, such a tree does not necessarily exist.
In this setting, the goal is to predict the labels with probability at least $1 - \epsilon - \epsilon_{T^*}$, where $\epsilon_{T^*}$ is the smallest prediction error among all trees on the set of points.

In the PAC-learning setting, our main result is as follows:
\begin{theorem}[Informal version of Theorem~\ref{thm:natarajan_main}]
\label{thm:pac-learning-intro}
    The sample complexity of $(\epsilon, \delta)$-PAC-learning hierarchically labeled tuples is $\Theta(\frac{n}{\epsilon} \cdot \polylog(\frac{1}{\epsilon}, \frac{1}{\delta}))$ in the realizable settings.
    Furthermore, in the agnostic setting, the complexity is $\Theta(\frac{n}{\epsilon^2} \cdot \polylog(\frac{1}{\epsilon}, \frac{1}{\delta}))$. 
\end{theorem} 
This result also holds for non-binary trees.


\paragraph{Online learning}

In the \emph{online learning} setting, it is no longer assumed that the tuples are generated from a fixed distribution.
The tuples can be selected to arrive in an adversarial order.
The accuracy is evaluated sequentially, counting the overall number of mistakes made by the algorithm throughout the sequence.
In the realizable setting, it is assumed that the tree generating the labels is fixed in advance and all labels seen by the algorithm are consistent with this tree.
In the agnostic setting, this is no longer assumed, similarly to the PAC-learning scenario.
Our main result in this setting is:

\begin{theorem}[Informal\arxiv{ version of Theorem~\ref{thm:littlestone_main_formal}}]
\label{thm:online_main}
On a sequence of triplets of length $T$, the number of mistakes made by an online learning algorithm for predicting their hierarchical structure is $\Theta(n \log n)$ in the realizable case. In the agnostic case, the number of mistakes is at most $O(\sqrt{T n \log n  \log (T)})$ and at least $\Omega(\sqrt{T n \log n})$ larger than the optimum number of mistakes achieved by any tree.
\end{theorem}

A summary of our results is shown in Table~\ref{table:results} (which hold for both triplet queries, and more generally to $k$-tuples for constant $k$). 

\def\arraystretch{1.5}
\begin{table}[ht!]
    \centering
    \begin{tabular}{ |l|c| }
        \hline
         & Number of $k$-tuple labels for constant $k$ \\
        \hline
        PAC realizable* & $\Theta(\frac{n}{\epsilon} \cdot \polylog(\frac{1}{\epsilon}, \frac{1}{\delta}))$ \\ 
        \hline
        PAC agnostic* & $\Theta(\frac{n}{\epsilon^2} \cdot \polylog(\frac{1}{\epsilon}, \frac{1}{\delta}))$ \\ 
        \hline
        Online realizable**  & $\Theta(n \log n)$ \\
        \hline
        Online agnostic** & \makecell{$O(\sqrt{T n \log n  \log (T)})+\mathrm{OPT} $, \\ $\Omega(\sqrt{T n \log n}) + \mathrm{OPT}$} \\
        \hline
    \end{tabular}
    \ifarxiv{\caption{[*] Binary trees are considered in Theorem~\ref{thm:natarajan_main}, and non-binary trees in Theorem~\ref{thm:non-binary}. [**] Binary trees are considered in Theorem~\ref{thm:littlestone_main_formal}, and the non-binary case follows from Theorem~\ref{thm:non-binary}.}\label{table:results}}\else{\caption{Summary of sample complexities in each setting}\label{table:results}}\fi
\end{table}
\def\arraystretch{1}

\subsection{Our Techniques}
\label{ssec:our_techniques}
It is known (see e.g.~\citet{DanielySBS15}) that the PAC-learning and online learning complexities are nearly tightly described in terms of the Natarajan~\citep{N89} and Littlestone~\citep{Littlestone87} dimensions respectively. Hence in this paper we focus on giving tight bounds on these two values. Since the exact definitions are technical, we refer the reader to to Definition~\ref{def:natarajan} for the precise definition of Natarajan dimension and the full version for the Littlestone dimension.


\paragraph{Natarajan dimension}

In the PAC-learning setting, VC-dimension~\citep{VC71} can be used to capture query complexity for binary classification problems.
The problem of labeling tuples considered in this paper corresponds to multiclass classification since even for triplets there are three ``odd one out'' labels possible.
A generalization of VC-dimension which captures this scenario is the Natarajan dimension~\citep{N89}. While it can be shown via a simple probabilistic argument that this dimension is $O(n \log n)$, we argue that the exact bound is in fact linear in the number of points, which implies Theorem~\ref{thm:pac-learning-intro}.

In this section, we focus on triplet constraints; \ifarxiv{Lemma~\ref{lem:tuple_to_tiplet} shows how to reduce the $k$-tuples to triplets. }\else{in the full version, we show how to reduce $k$-tuples to triplets. }\fi
In order to bound the sample complexity, we use a version of the definition of Natarajan dimension adapted to our setting. 
Given a triple $\Delta = (a, b, c)$, we denote an ``odd one out'' constraint separating $c$ from $a$ and $b$ as $\C abc$.
For each triplet, there are $3$ possible constraints, i.e. $3$ possible labels.
We are now ready to define the notion of Natarajan Shattering.
\begin{definition}[Natarajan Shattering]
\label{def:natarajan_shattering}
Let $S = \{(a_1, b_1, c_1), \dots (a_k, b_k, c_k)\}$ be a set of triples of points. We say $S$ is \emph{Natarajan shattered} if for every triple $(a_i, b_i, c_i) \in S$ there exist two distinct constraints $f_1(a_i, b_i, c_i)$, $f_2(a_i, b_i, c_i)$ (e.g. $f_1(a_i, b_i, c_i) = \C{a_i}{c_i}{b_i}, f_2(a_i, b_i, c_i) = \C{a_i}{b_i}{c_i}$) on this triple, such that for every subset $R \subseteq S$ there is a hierarchical tree $T$ for which:
\begin{compactitem}
    \item For every $\Delta \in R$ it holds that $T$ satisfies $f_1(\Delta)$,
    \item For every $\Delta \in S \setminus R$ it holds that $T$ satisfies $f_2(\Delta)$.
\end{compactitem}
The Natarajan dimension of the hierarchically labeled tuples is defined as the size of the largest cardinality of a set which can be Natarajan shattered.
\end{definition}
Intuitively, for every triplet in $S$, we fix two labels (out of three possible).
The set is shattered if all possible $2^{|S|}$ combinations of these labels on $S$ are realizable by the hypothesis space.

\begin{example}
Consider a point set, consisting of $4$ points: $a, b, c, d$. Consider two sets of triples:
\begin{itemize}
    \item $S_1 = \{(a, b, c), (b, c, d)\}$
    \item $S_2 = \{(a, b, c), (b, c, d), (c, d, a)\}$.
\end{itemize}
Note that for $S_1$ we can select $f_1(a, b, c) = \C abc$, $f_2(a, b, c) = \C bca$ and also $f_1(b,c,d) = \C bcd, f_2(b, c, d) = \C cdb$. Then for any of the four subsets of $S_1$ we can create a tree which is consistent with the $f_1$ choices on the subset and consistent with the $f_2$ choices otherwise. We show all four resulting trees on Figure~\ref{fig:shattering}.
\begin{figure}[t!]
	\centering
        \
\newcommand{\treeparams}{for tree={s sep=0.6cm,l-=1.5em,solid node}}
\bracketset{action character=@}

\begin{tikzext}
    \begin{forest}
        [ ,@\treeparams,
          [
            [
              [ ,label={below:$a$} ]
              [ ,label={below:$b$} ]
            ]
            [ ,label={below:$c$}]
          ]
          [ ,label={below:$d$}]
        ];
        \node[below=8em,align=center,anchor=center] {$\C abc$ \\ $\C bcd$};
    \end{forest}
    \quad
    \begin{forest}
        [ ,@\treeparams,
          [
            [
              [ ,label={below:$b$} ]
              [ ,label={below:$c$} ]
            ]
            [ ,label={below:$a$}]
          ]
          [ ,label={below:$d$}]
        ];
        \node[below=8em,align=center,anchor=center] {$\C bca$ \\ $\C bcd$};
    \end{forest}
    \quad
    \begin{forest}
        [ ,@\treeparams,
          [
            [ ,label={below:$a$} ]
            [ ,label={below:$b$} ]
          ]
          [
            [ ,label={below:$c$}]
            [ ,label={below:$d$}]
          ]
        ];
        \node[below=8em,align=center,anchor=center] {$\C abc$ \\ $\C cdb$};
    \end{forest}
    \quad
    \begin{forest}
        [ ,@\treeparams,
          [
            [
              [ ,label={below:$c$} ]
              [ ,label={below:$d$} ]
            ]
            [ ,label={below:$b$}]
          ]
          [ ,label={below:$a$}]
        ];
        \node[below=8em,align=center,anchor=center] {$\C bca$ \\ $\C cdb$};
    \end{forest}
\end{tikzext}
	\caption{Shattering a set of two triplets $\{(a, b, c), (b, c, d)\}$.}
		\label{fig:shattering}
\end{figure}
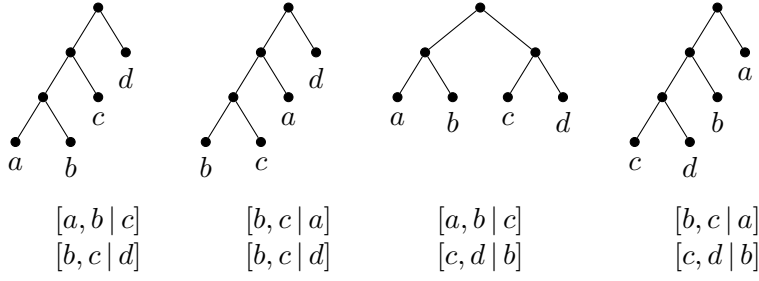
It is less straightforward, however, to check that for the set $S_2$, no possible selection of $f_1, f_2$ for each of its triangles can be used to satisfy the definition of Natarajan shattering.
This follows from our key technical result.
\end{example}

\begin{theorem}\label{thm:natarajan-intro}
The Natarajan dimension of hierarchically labeled triplets on $n$ elements is $n - 2$.
\end{theorem}

The proof is based on identifying subsets with a certain property, which we refer to as \emph{closed sets} (Definition~\ref{def:closed_set}) which prevent construction of a consistent tree.
We then further characterize these sets in terms of easier to define \emph{critical sets} (Definition~\ref{def:critical_set}).
We prove that for $n-1$ elements a critical set must exist, which, combined with the fact that a set of consistent $n-2$ pairs of constraints exists, gives a tight bound for the Natarajan dimension.

\paragraph{Littlestone dimension}
Next, we show a tight bound on the Littlestone dimension. For the purpose of the introduction, it is convenient to view the Littlestone dimension in our setting as follows. Assume that we have two players Alice and Bob who play the following game for $t$ iterations. Initially, there is an empty set of constraints $C$. In every iteration, Alice passes Bob a triplet  from the set of points $V$, and two distinct elements from this triplet. For example,  suppose that the triplet is $(a,b,c)$ and the elements chosen are $a,b$. Then Bob chooses one of the two constraints where one of the chosen elements is the ``odd one out'', i.e.  $\C bca$ or $\C acb$ and adds it to $C$.
The Littlestone dimension is the maximum number of rounds $t$ for which Alice has a strategy that results in $C$ admitting a tree that satisfies all constraints in $C$, regardless of Bob's strategy.


\begin{theorem}
\label{thm:littlestone_intro}
The Littlestone dimension of hierarchically labeled triples is $\Theta(n \log n)$.
\end{theorem}

To prove the lower bound on the Littlestone dimension, we need to describe Alice's strategy of producing an adaptive sequence of queries of size $\Omega(n\log{n})$ such that for any choice of Bob's query answers there is a hierarchical tree consistent with the query answers at the end of the sequence. The sequence we construct is intuitively a sort tournament on $V$, for which there is always some ordering $v_1,\dots,v_n$ on the point set $V$ such that if we place the points in order in a hierarchical tree whose internal nodes are shaped like a path, this tree satisfies all constraints.


\paragraph{Non-binary trees} When all constraints are of type ``odd one out'', it suffices to consider binary trees: if a non-binary tree is consistent with the constraints, then its binarization (i.e. we replace a node with $\ell$ children with an arbitrary binary tree on these children) is also consistent.
However, one may generalize the problem to non-binary trees by considering constraints of the form \textit{``points $i,j$ and $k$ must be split simultaneously''}.
For the PAC-learning setting, in \ifarxiv{Appendix~\ref{app:missing_proofs}}\else{the full version}\fi, we consider and solve this setting by extending Theorem~\ref{thm:pac-learning-intro} appropriately.
In the online learning setting, Theorem~\ref{thm:online_main} generalizes as well -- the same lower bound holds for the non-binary setting and the upper bound generalizes trivially.


\subsection{Related Work}
An early work by~\citet{ASSU81} shows that given access to $m$ consistently labeled triples on $n$ vertices, the tree satisfying them can be constructed in $O(mn)$ time. This was improved to $O(m \log^2 n)$ time using the techniques in~\citet{HKW99, HLT01}, and to $O(\min(n^2,m\log{n})+\sqrt{mn}\log^{2.5}(n))$ by~\citet{Thorup99}.

Compared with our settings where we are given queries from some distribution and aim for a bounded error rate on unseen triplets, a related line of work considers the problem of exactly reconstructing the entire tree using adaptively or non-adaptively chosen triplet queries (see e.g.~\citet{KLW96,EK18}).
For non-adaptive queries, a lower bound of $\Omega(n^3)$ queries precludes any non-trivial results~\citep{EK18}.
For adaptive queries, $\Theta(n \log n)$ consistent queries and running time are necessary and sufficient for the construction of the tree~\citep{KLW96}.
This can be extended to handle a mix of independently correct labels and adversarial noise, resulting in similar bounds for constant noise levels~\citep{EK18}.

In the case when labeled data is allowed to be inconsistent, minimizing the number of disagreements with labeled triplets is known to be hard to approximate~\citep{CDW13}.
Recent work~\citet{CMA21} gives algorithms for maximizing the number of agreements between the tree and the labeled data. In a related line of work, results are known for optimizing certain objectives while respecting triplet constraints~\citep{CNC18}.

\section{Preliminaries}
\label{sec:preliminaries}
\noindent We begin with a formal definition of a hierarchical tree.
\begin{definition}
    \label{def:hc_tree}
    Given a set of points $V$, we define a \emph{hierarchical tree} $T$ as a binary tree such that $V$ is bijectively mapped on the leaves of $T$.
\end{definition}
With a slight abuse of notation, for a fixed tree $T$, we treat elements of $V$ as the leaves of $T$.
For two leaves $i,j$ and a hierarchical tree $T$, we denote the least common ancestor, i.e. the internal node corresponding to the smallest subtree containing both $i$ and $j$, as $\lca_T(i,j)$.

We are interested in satisfying structural constraints on the elements of $V$.
In each section, we consider different types of constraints in the following order:
\begin{itemize}
    \item We first consider the simplest structured constraints, i.e. constraints on three elements.
    \item We then generalize our result for constraints on $k \ge 3$ elements.
    \item Finally, we further generalize our result for the case when the trees are not necessarily binary.
\end{itemize}
\paragraph{Triplet constraints}
\begin{definition}
    \label{def:triplet_constraints} 
    For a hierarchical tree $T$ and a triplet of distinct points $(a,b,c) \subseteq V$\footnote{Since the original order inside the tuples doesn't matter, with a slight abuse of notation we treat tuples (triplets in particular) as sets.} we say that $T$ \emph{satisfies the constraint} denoted $[a,b | c]$ if $T$ cuts $c$ from $a$ and $b$, i.e. $\lca_T(a,c) = \lca_T(b,c)$.
    We call such a constraint an orientation of triplet $(a,b,c)$.
\end{definition}
\begin{definition}
    \label{def:orientation}
    For a triplet $t = (a,b,c) \subseteq V$ we say that an \emph{orientation} of the triplet is a constraint over the nodes $a,b,c$, i.e. $\C abc$, $\C acb$, or $\C bca$. We further denote it as $\orient{t}$. Given a set of triplets $\Delta = \{(a_i, b_i, c_i)\}_i$ over $V$, we say that an orientation of $\Delta$ is a set of orientations over each triplet in $\Delta$. We similarly denote it as $\orient{\Delta}$.
\end{definition}
%
That is, orientation is a particular choice of constraint(s) generated based on a given triplet(s).
In order to apply our results in the PAC-learning setting, we are interested in the sets of constraints that can not be satisfied:
\begin{definition}
    \label{def:contradictory_triplet_set}
    We define a set of constraints as \emph{contradictory} if there is no hierarchical tree that satisfies all constraints in the set.
\end{definition}

\paragraph{$k$-tuple constraints}

Any orientation of a triplet uniquely defines a tree on this triplet.
We can use this intuition to define constraints on $k$ elements.
\begin{definition}
    \label{def:k_constraint}
    Let $A=(a_1, \ldots, a_k)$ be a $k$-tuple, and let $T_A$ be a binary tree with leaves $a_1, \ldots, a_k$.
    Then we say that a binary tree $T$ \emph{satisfies constraint $T_A$} if any triplet constraint $\C{a_i}{a_j}{a_t}$ satisfied by $T_A$ is also satisfied by $T$.
\end{definition}
\begin{definition}
    \label{def:k_tuple_orientation}
    For a given $k$-tuple $(a_1, \ldots, a_k) \subseteq V$, an orientation of the tuple is any constraint on $a_1,\ldots, a_k$.
\end{definition}

\paragraph{Non-binary trees}
When non-binary trees are allowed, some nodes $i,j,k$ can be separated at the same time,
i.e. $\lca_T(a,b)=\lca_T(a,c)=\lca_T(b,c)$.
We denote such case as $\TW abc$.
Similarly, we allow $k$-tuple constraints where multiple elements can be separated at the same time.

\paragraph{Sample complexity}
Let $\mathcal D$ be the distribution on $X \times Y$, where $X$ is the set of inputs and $Y$ is the set of labels.
Let $H \subseteq Y^X$ be a hypothesis space\footnote{$Y^X$ is the set of all functions $X \to Y$}.
For a given $h \in H$, we define the error rate as $\err_{\mathcal D}(h) = \mathbb{P}_{(x,y) \sim \mathcal D} [h(x) \ne y]$.
Let $h^*_{\mathcal D} = \argmin_{h \in H} \err_{\mathcal D}(h)$.
We say that the settings are realizable if $\err_{\mathcal D}(h^*_{\mathcal D}) = 0$; otherwise, we say that the settings are agnostic.
\begin{definition}[Sample complexity]
    We define \emph{sample complexity} $m_H(\eps, \delta)$ as the minimum number of samples, such that there exists a predictor that, for any distribution $\mathcal D$, given $m_H(\eps, \delta)$ samples from $\mathcal D$, achieves error rate at most $\err_{\mathcal D}(h^*_{\mathcal D}) + \epsilon$ with probability at least $\delta$.
    We denote the sample complexity as $m_H^r(\eps, \delta)$ for the realizable case and $m_H^a(\eps, \delta)$ for the agnostic case.
\end{definition}
For the binary classification task ($|Y|=2$), the sample complexity can be estimated using VC-dimension~\citep{VC71}.
Its analog for the multi-class settings is the Natarajan dimension.
\begin{definition}[Natarajan dimension~\citep{N89}]\label{def:natarajan}
    Let $X$ be the set of inputs, $Y$ be the set of labels, and let $H \subseteq Y^X$ be a hypothesis class.
    We say that $S \subseteq X$ is N-shattered by $H$ if there exist $f_1,f_2\colon\ X \to Y$ such that $f_1(x) \ne f_2(x)$ for all $x \in S$ and for every $T \subseteq S$ there exists $g \in H$ such that:
    \[
        g(x)=f_1(x) \text{ for } x \in T\text{ and } g(x)=f_2(x)\text{ for } x \notin T
    \]
    The Natarajan dimension $\ndim(H)$ of $H$ is the maximum size of an N-shattered set.
\end{definition}
In our case, the hypothesis space is defined by the set of constraints induced by all possible hierarchical trees:
\begin{definition}
    \label{def:hypothesis}
    Given a set $V$ be a set and an integer $k \ge 3$, let $X$ be the set of $k$-tuples on $V$ and $Y$ be the set of orientations of the $k$-tuples.
    Then we use $H_k(V)$ to denote a set of mappings $X \to Y$ such that for each mapping there exists a hierarchical tree where each $k$-tuple is oriented according to the mapping.
\end{definition}
The following result gives a tight estimate of the sample complexity based on the Natarajan dimension of the problem
\begin{lemma}[\citet{bendavid1995characterizations}]
    \label{lem:natarajan_samples}
    If $|Y| < \infty$, then for the sample complexity $m_H^r(\eps, \delta)$, the following holds for some universal constants $C_1$ and $C_2$ for the realizable case:
    \begin{align*}
        C_1 \frac{\ndim(H) + \log \frac 1\delta}{\eps} &\le m_H^r(\eps, \delta) \le C_2 \frac{\ndim(H) \log \frac 1\eps \log |Y| + \log \frac 1\delta}{\eps}
    \shortintertext{For the agnostic case:}
        C_1 \frac{\ndim(H) + \log \frac 1\delta}{\eps^2} &\le m_H^a(\eps, \delta) \le C_2 \frac{\ndim(H) \log |Y| + \log \frac 1\delta}{\eps^2}
    \end{align*}
\end{lemma}
\todo{Littlestone dimension}





\section{Contradictory orientations}
\label{sec:contr_orient}

As described in Section~\ref{ssec:our_techniques}, the key component in our analysis is a tight bound on the Natarajan dimension.
To find it, we first consider a simpler question: ``\textit{for a given $n$, what is the minimum $m$, such that for any set of $m$ triplets on $[n]$ there exists a contradictory orientation of these triplets?}''
Recall that the definition of Natarajan dimension restricts the candidate orientations so that every triplet has only two allowed orientations, giving $2^m$ possible label combinations.
In this section, we answer this question without such a restriction, i.e. we check whether all $3^m$ label combinations are possible\footnote{One may think of this as another generalization of VC shattering, although not directly applicable to PAC learning}, and handle the restriction in the next section.
In what follows we prove that $m=n-1$.
\begin{theorem}
    \label{thm:triplets}
    For any $n > 2$, any set of triplets of size at least $n-1$ on these points allows for a contradictory orientation.
\end{theorem}

\subsection{Closed set}

We can think about every constraint $\C abc$ as an edge $(a,b)$ which corresponds to a separate vertex $c$, and we say that $\C abc$ ``generates'' edge $(a,b)$.
Clearly, if a hierarchical tree satisfies constraint $\C abc$, then there must exist a tree node such that one child's subtree contains $a$ and $b$, and another child's subtree contains $c$.
In other words, the tree is contradictory if it cuts an edge $(a,b)$ before first separating $c$ from $a$ and $b$.
When building a tree in a top-down manner, each node corresponds to some set of elements $S \subseteq V$, which we want to partition further.
When partitioning $S$, only triplets lying entirely in $S$ (which we call induced by $S$) can lead to a contradiction.
%
\begin{definition}[Induced triplets/constraints]
    Let $\Delta$ be a set of triplets over $V$ and let $S \subseteq V$. We define the set of triplets in $\Delta$ which are \emph{induced} by $S$ as $\Delta |_{S} = \{t \in \Delta \mid t \subseteq S\}$ (i.e., the set of all triplets that lie entirely within $S$).
    Similarly, given a set of constraints $C$, we define $C |_{S} = \bbr{\C abc \in C \mid a,b,c \in S}$.
\end{definition}
The above reasoning implies that it's impossible to split a tree node with set $S$ if it's connected by the edges generated by the constraints induced by $S$, since splitting $S$ would cut at least one edge (see Figure~\ref{fig:connected_contradiction} for example).
We will show that the existence of the hierarchical tree is determined by the existence of such set $S$.
For fixed $\Delta$, if it's possible to orient $\Delta$ so that $\orient{\Delta}|_S$ connects $S$, we call $S$ a closed set.
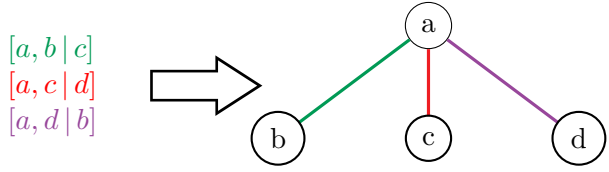
\begin{figure}[t!]
    \centering
\begin{tikzpicture}[-,level/.style={sibling distance = 2cm/#1, level distance = 1.5cm}]
    \node[align=left] at (-5,-0.8) {{\color{ForestGreen}$\C abc$}\\
                                    {\color{red}$\C acd$}\\
                                    {\color{Purple}$\C adb$}};
    \draw[line width=12pt,my triangle,
            postaction={draw,white,line width=10pt,my triangle,shorten >=2pt,shorten <=1pt}](-3.7,-0.8) -- (-2.2,-0.8);
    \node [circle, black, draw, fill=white] {a}
        child{ [ForestGreen, very thick] node [circle, black, draw, fill=white, thick] {b} }
        child{ [Red, very thick] node [circle, black, draw, fill=white, thick] {c} }
        child{ [Purple, very thick] node [circle, black, draw, fill=white, thick] {d} }
    ; 
\end{tikzpicture}
    \caption{Let $S = \{a,b,c,d\}$ and $\orient{\Delta} = \bbr{\C abc, \C acd, \C adb}$.
    Since $S$ is connected by edges generated by $\orient{\Delta}|_S$, the constraints are contradictory.
    When first splitting $S$, we must cut an edge, which leads to a contradiction: cutting edge $(a,b)$ violates $\C abc$, cutting $(a,c)$ violates $\C acd$, and cutting $(a,d)$ violates $\C adb$.
    }
    \label{fig:connected_contradiction}
\end{figure}
\begin{definition}[Closed set] \label{def:closed_set}
    Let $\Delta$ be a set of triplets over $V$. We say that a set $S \subseteq V $ is \emph{closed} w.r.t. $\Delta$ if there exists an orientation of $\Delta$ denoted as $\orient{\Delta}$, such that $E|_S=\{(a,b) \mid \C abc \in \orient{\Delta} |_{S}\}$ connects $S$.
\end{definition}
%



\noindent The following Lemma shows that the presence of a closed subset is a necessary and sufficient condition for existence of a contradictory orientation.
The proof is given in \ifarxiv{Appendix~\ref{app:missing_proofs}}\else{the full version}\fi. 
\begin{lemma}
    \label{lem:closed_set}
    A set of triplets $\Delta$ over $V$ allows for a contradictory orientation if and only if there exists a set $S \subseteq V$ that is closed w.r.t. $\Delta$.
\end{lemma}
The proof is based on the above intuition.
If the set is not connected by the edges generated by $\orient{\Delta} |_{S}$, then by separating connected components we don't cut any edge, and hence don't violate any constraints.
On the other hand, if the set is connected, when we first split $S$, we cut at least one edge $(a,b)$ such that $\C abc \in \orient{\Delta}|_S$.
But this violates the constraint, as it requires $c$ to be first separated from $a$ and $b$.

\subsection{Critical set}

We next show that, when $\Delta$ is sufficiently large, there exists a set $S$ such that we can always construct a contradictory orientation of $\Delta|_S$.
We call such a set a critical set.
\begin{definition}[Critical set]
    \label{def:critical_set}
    Let $\Delta$ be a set of triplets over $V$ of size $|\Delta| \ge |V| - 1$.
    We say that set $S \subseteq V$ is \emph{critical} w.r.t. $\Delta$ if it satisfies the following conditions:
    \begin{itemize}
        \item $S$ induces at least $|S|-1$ triplets, i.e. $\card{\Delta |_S} \geq |S|-1$,
        \item Among all such sets, $S$ has the minimal cardinality.
    \end{itemize}
\end{definition}


\noindent Note that this is well defined since $V$ satisfies the condition: $|\Delta| = \card{\Delta |_V} \geq |V|-1$.
Intuitively, $|S| - 1$ is the minimum possible number of edges which can connect $S$.
The surprising fact is that this condition suffices, which leads to the main result of this section (Theorem~\ref{thm:triplets}).
\begin{theorem}
    \label{thm:crit_set_is_a_closed_set}
    Let $\Delta$ be a set of triplets over $V$ of size $|\Delta| \geq |V| -1$.
    Then any critical set w.r.t. $\Delta$ is closed w.r.t. $\Delta$.
\end{theorem}
We present the full proof of Theorem~\ref{thm:crit_set_is_a_closed_set} \ifarxiv{in Appendix~\ref{app:missing_proofs}}\else{in the full version}\fi.
The proof relies on the operation which we call \emph{reorientation}: given constraint $\C abc$, we change it to either $\C acb$ or $\C bca$, which respectively changes the generated edge from $(a,b)$ to either $(a,c)$ or $(b,c)$.

The proof is by contradiction: for a critical set $S$, we assume that no orientation of $\Delta|_S$ generates edges connecting $S$.
Among all orientations, we consider the one whose edges result in the smallest number of connected components.
We show that by performing certain reorientations, we can reduce the number of connected components, leading to a contradiction.

Since the edges generated by the orientation don't connect $S$, there must exist an ``unused'' triplet $(a,b,c)$, such that adding or removing its edges $(a,b)$, $(a,c)$ and $(b,c)$ doesn't change the number of connected components.
All of $a,b,c$ belong to the same tree $T$ in the spanning forest of $S$ (otherwise, we could orient $(a,b,c)$ so that it connects two trees, reducing the number of connected components).
By orienting this triplet as, for example, $\C abc$, we can reorient any edge on the path from $a$ to $b$ in $T$ without disconnecting vertices in $T$.
If there exists a reorientation that connects $T$ to another tree in the spanning forest, using this reorientation we decrease the number of connected components, as shown in Figure~\ref{fig:reorient_simple}.
Such a reorientation might not always be immediately available, but we show that it's always possible to build a chain of reorientations so that the number of connected components decreases.
\begin{figure}[t!]
    \centering
    \ifarxiv
        \includegraphics[width=0.6\columnwidth]{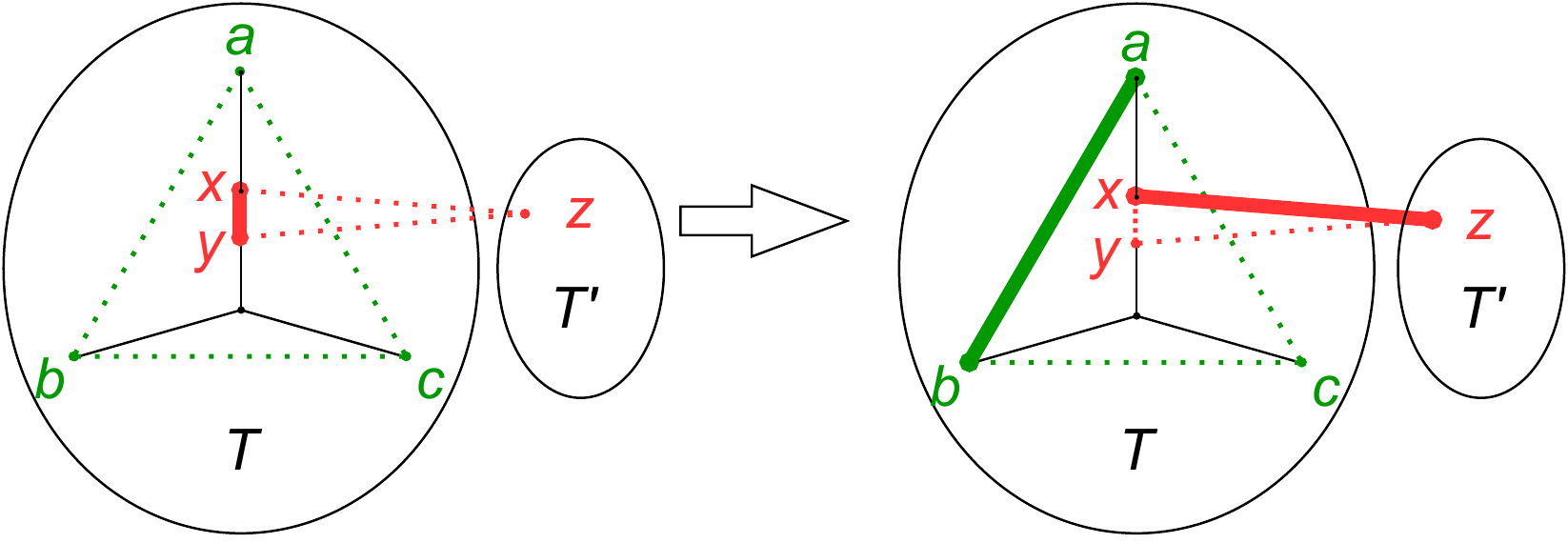}
    \else
        \includegraphics[width=0.8\columnwidth]{pics/Triple_inside_copy.pdf}
    \fi
    \caption{A case from Theorem 8: selecting orientation $\C abc$ allows one to reorient any edge on the path from $a$ to $b$. In this case, on this path there is an edge $(x, y)$ generated by constraint $\C xyz$, where $z$ belongs to another spanning tree. Reorienting $\C xyz$ connects the trees}
    \label{fig:reorient_simple}
\end{figure}


\paragraph{$k$-tuples}
In \ifarxiv{Theorem~\ref{thm:tuples}}\else{the full version}\fi, we show that any set of $k$-tuples of size at least $\left\lceil \frac{|V|-1}{k-2} \right\rceil$ has a contradictory specification\arxiv{,} and \arxiv{Theorem~\ref{thm:shatter_lb} shows} that the bound is tight for $k$-tuples\footnote{However, as we show in the next section, the $(k-2)^{-1}$ factor doesn't propagate to the sample complexity}.
The main idea is that, for the sake of analysis, we can replace a $k$-tuple with $k-2$ ``independent'' triplets, meaning that all $3^{k-2}$ orientations of these triplets are not contradictory.
Namely, all the triplets share two elements and differ in the third one.

\paragraph{Non-binary trees}
When constraints of form $\C abc$ are allowed, after adjusting the definitions, the overall idea is the same: a contradictory orientation exists iff a closed set exists, and any critical set is closed.
While the definition of a critical set doesn't change~-- it's a minimum-size set $S$ with $\ge |S|-1$ induced edges~-- the definition of a closed set changes substantially, as described next.

As before, we have a set $\mathcal E = E|_S$ of edges induced by $S$, but now we might have additional constraints of form $[a|b|c]$.
The main idea is as follows: if e.g. $a$ and $b$ are already connected by $\mathcal{E}$, they can't be separated by the first cut.
Hence, by definition of $[a|b|c]$, $c$ also can't be separated from $a$ and $b$ by the first cut.
Since we perform the first cut based on the connectivity of the set, the fact that $c$ can't be separated from $a$ and $b$ can be expressed by adding edge $(a,c)$, i.e. $\mathcal E \gets \mathcal E \cup \{(a,c)\}$, hence connecting $a$, $b$ and $c$ (see Figure~\ref{fig:nonbinary_edge}).
\begin{figure}[t!]
    \centering
\begin{tikzext}
    \begin{tikzpicture}
        \node (a) at (0,0) {$a$};
        \node (b) at (0, -1) {$b$};
        \node (c) at (-0.86, -0.5) {$c$};
    
        \draw [rounded corners=10pt] ($(c)+(-0.5,0)$) -- ($(a)+(0.3,0.5)$) -- ($(b)+(0.3,-0.5)$) --cycle;
    
         \draw[decorate, decoration={snake, amplitude=0.5mm}] ($(b) + (0.1,0)$) arc(-90:90:2.15 and 0.5);
    
        \draw[line width=12pt,my triangle,
            postaction={draw,white,line width=10pt,my triangle,shorten >=2pt,shorten <=1pt}](3,-0.5) -- (5,-0.5);
        
        \node (a2) at (6.5,0) {$a$};
        \node (b2) at (6.5, -1) {$b$};
        \node (c2) at (6.5 - 0.86, -0.5) {$c$};
    
    
         \draw[decorate, decoration={snake, amplitude=0.5mm}] ($(b2) + (0.1,0)$) arc(-90:90:2.15 and 0.5);
         \draw (a2) -- (c2);
    \end{tikzpicture}
\end{tikzext}
    \caption{Given constraint $[a|b|c]$, when $a$ and $b$ are connected, we can also connect them to $c$}
    \label{fig:nonbinary_edge}
\end{figure}
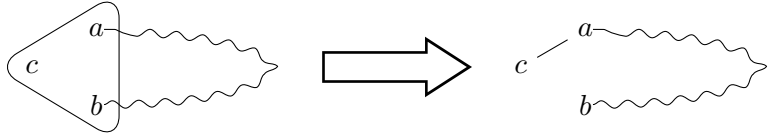

Now, we can say that $S$ is \emph{closed} w.r.t. $\Delta$ if there exists an orientation of $\Delta$ such that $S$ is connected after performing all such possible operations.
Similarly to the binary-tree case, the existence of such a set implies a contradictory orientation.
This intuition is formalized in \ifarxiv{Lemma~\ref{lem:nonbinary:closed_is_contradictory}}\else{the full version}\fi.

Obviously, critical sets are closed: by only using constraints of type $\C abc$, we can use Theorem~\ref{thm:crit_set_is_a_closed_set} directly.
However, constraints of type $[a|b|c]$ result in an additional option in the definition of N-shattering (we have to choose $2$ labels out of $4$ instead of $3$), and hence the Natarajan dimension can potentially increase.
In the next section, we show that this is not the case for our problem.

\section{PAC-learning and Natarajan Dimension}
\label{sec:pac}

In this section, we present tight sample complexity bounds for PAC learning for $k$-tuple constraints.
Recall that the set $\Delta = \{t_i\}_{i=1}^n$ of $k$-tuples is N-shattered if for every $k$-tuple $t_i$ we can select two different orientations $\orient{t}^{(1)}_i$ and $\orient{t}^{(2)}_i$ such that every combination of orientations of different $k$-tuples is not contradictory, i.e. for any $f:\ [n] \to [2]$ the orientation $\{\orient{t}^{(f(i))}_i\}_{i=1}^n$ is not contradictory.
Given the Natarajan dimension, Lemma~\ref{lem:natarajan_samples} gives the tight bound on sample complexity up to the factor $O(\frac 1\eps \log |Y|)$.
Note that when $k$ is constant, $|Y|$ is also constant.

We first lower-bound the Natarajan dimension:
\begin{lemma}\label{lem:natarajan_lb}
    For any $V$ and $k \ge 3$, we have $\ndim(H_k(V)) \ge |V| - k + 1$.
\end{lemma}
\begin{proof}
	Let $A$ be an arbitrary subset of $V$ of size $k-1$, and let $B = V \setminus A$.
    We construct the set of $k$-tuples as $\Delta = \{A \cup \{b\} \mid b \in B\}$.
    Let $T_A$ be an arbitrary hierarchical tree on $A$.
    For each $A \cup \{b\} \in \Delta$, we construct the orientations as follows: we replace a leaf $a$ of $T_A$ with a new node with two children: $a$ and $b$.
    By choosing two different leaves, we obtain two different orientations of $A \cup \{b\}$, as required by the definition of N-shattering.
    
    It's easy to check that $\Delta$ is N-shattered using these orientations: all elements from $A$ are in agreement across all constraints since they all are oriented according to $T_A$, while every element from $B$ participates in only one constraint, and hence can't lead to a contradiction.
    Therefore, $\ndim(H_k(V)) \ge |B| = |V| - k + 1$.
\end{proof}

\noindent Next, we upper-bound the Natarajan dimension for triplets using results from Section~\ref{sec:contr_orient}.
\begin{theorem}\label{thm:natarajan}
    For any $V$, we have $\ndim(H_3(V)) = |V|-2$.
\end{theorem}
We provide the full proof in \ifarxiv{Appendix~\ref{app:n_dimension}}\else{the full version}\fi.
First, note that the result doesn't immediately follow from Theorem~\ref{thm:crit_set_is_a_closed_set} since Theorem~\ref{thm:crit_set_is_a_closed_set} finds a contradictory orientation among all possible $3^{m}$ orientations of $m$ constraints.
However, Natarajan shattering allows us to choose one of only two orientations of each triplet, i.e. it allows $2^{m}$ possible orientations.
Hence, we need to handle the case when the contradicting orientation from Theorem~\ref{thm:crit_set_is_a_closed_set} orients some triplet in a non-allowed way.

Let $S$ be a critical set, $\orient{\Delta}|_S$ be its induced orientation connecting $S$, and $[a,b | c] \in \orient{\Delta}|_S$ be a non-allowed orientation of a triplet.
Note that it means that both remaining orientations $[a,c| b]$ and $[b,c | a]$ are allowed.
If removing edge $(a,b)$ doesn't disconnect $S$, we reorient it arbitrarily.
Otherwise, removing $(a,b)$ separates $S$ into two connected components.
We reorient $(a,b)$ in the following way: if $c$ belongs to the same part as $a$, then we use edge $(b,c)$; otherwise, we use edge $(a,c)$.
As shown in Figure~\ref{fig:reorient_natarajan}, such reorientation preserves connectivity of $S$.
\begin{figure}[t!]
    \centering
    \ifarxiv
        \includegraphics[width=0.6\columnwidth]{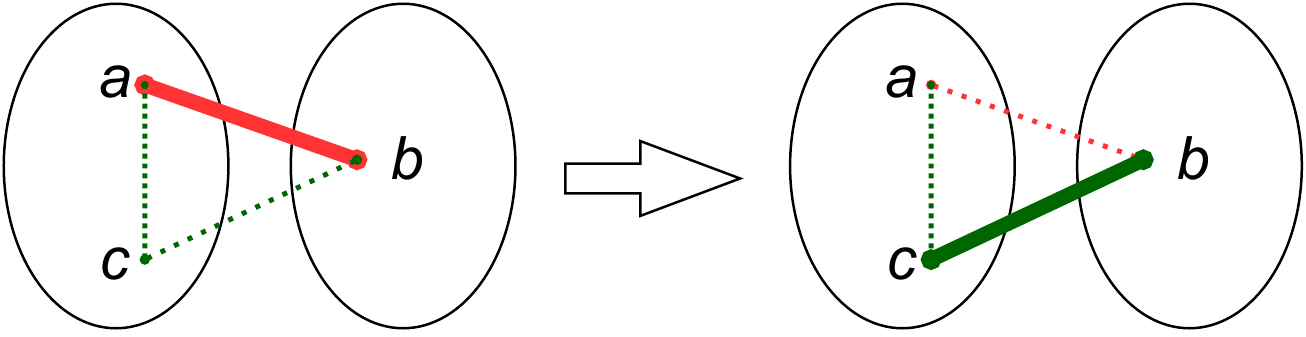}
    \else
        \includegraphics[width=0.8\columnwidth]{pics/Reorient_natarajan.pdf}
    \fi
    \caption{Theorem~\ref{thm:natarajan}: when orientation $\C abc$ is not allowed, we can reorient it while preserving connectivity}
    \label{fig:reorient_natarajan}
\end{figure}

Combining Lemma~\ref{lem:natarajan_lb} and Theorem~\ref{thm:natarajan} yields the following corollary.
\begin{corollary}
    For any $V$ and $k \ge 3$, we have $|V| - k + 1 \le \ndim(H_k(V)) \le |V| - 2$.
\end{corollary}
Using Lemma~\ref{lem:natarajan_samples}, we get our main result.
%
\begin{theorem}
\label{thm:natarajan_main}
    For constant $k$, the sample complexity of learning hierarchically labeled $k$-tuples, denoted by $m_{H_k}^r(\eps, \delta)$ in the realizable setting and $m_{H_k}^a(\eps, \delta)$ in the agnostic setting, is bounded by:
    \begin{align*}
        C_1 \frac{n + \log \frac 1\delta}{\eps} \le m_{H_k}^r(\eps, \delta) & \le C_2 \frac{n \log \frac 1\eps + \log \frac 1\delta}{\eps} \\
        C_1 \frac{n + \log \frac 1\delta}{\eps^2} \le m_{H_k}^a(\eps, \delta) & \le C_2 \frac{n + \log \frac 1\delta}{\eps^2}
    \end{align*}
\end{theorem}

In \ifarxiv{Appendix~\ref{app:n_dimension}}\else{the full version}\fi, we prove that the identical result holds for the non-binary case.
The proof idea is similar to that of Theorem~\ref{thm:natarajan}, with the only change that we must handle the case when $[a|b|c]$ is one of the allowed constraints.
\section{Experiments}
\label{sec:experiments}

In this section, we empirically verify our theoretical findings by evaluating the prediction accuracy of binary trees constructed from labeled triplets on the unlabeled triplets from the same distribution.
We consider the binary realizable case, when the triplets are labeled according to a ground-truth binary tree, the binary agnostic case, and the non-binary realizable case.

\paragraph{Tree building algorithm} We first describe in detail our approach for building a tree.
For the agnostic case, the constraints in our experiments are contradictory, since, as shown in Figure~\ref{fig:max_cons}, we encounter a contradiction after sampling $\approx 1.2 n$ constraints.
Hence, we need an approach that handles contradictory constraints.
It's known that the theoretical sample complexity can be achieved by using an Empirical Risk Minimizer~\citep{daniely2015multiclass}.
In our case, this means finding the tree which violates the least number of known constraints.
However, such a problem is NP-hard and is very hard to approximate~\cite{chester2013resolving}.

As shown in Lemma~\ref{lem:closed_set_proof}, the contradiction arises when the tree node encounters a connected component, and hence any partition cuts at least one edge.
When there are multiple connected components, we cut them from each other.
Otherwise, we cut the minimal number of edges at every layer, which is known to achieve $O(n)$ approximation.

\begin{figure}[t!]
    \centering
    \includegraphics[width=\ifarxiv{0.4}\else{0.7}\fi\columnwidth]{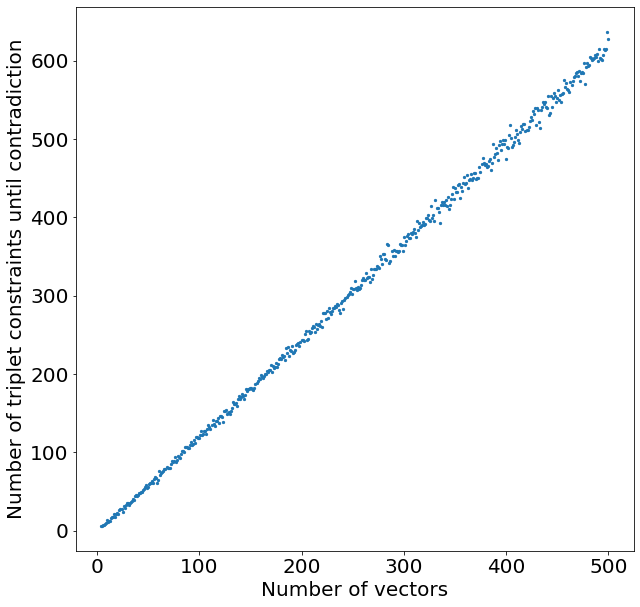}
    \caption{For the agnostic case with randomly generated vectors, for each number of vectors, we show the number of uniformly sampled triplets on these vectors (oriented according to the distances in the triplets) until we reach a contradiction, averaged over $10$ trials.
    There is a clear linear dependence with factor $\approx 1.2$}
    \label{fig:max_cons}
\end{figure}



\paragraph{Datasets}
For the realizable case, we perform experiments on randomly generated trees and on
ImageNet~\citep{deng2009Imagenet} hierarchy\footnote{\url{https://github.com/waitwaitforget/ImageNet-Hierarchy-Visualization}}.
We preprocess the ImageNet hierarchy as follows:
\begin{itemize}
    \item for non-binary tree experiments: the full hierarchy (48,860 points);
    \item for binary tree experiments: a sample of $256$ leaves that induce a binary subtree.
\end{itemize}

For the agnostic case, we consider 2 datasets: 1) randomly sampled vectors from the uniform distribution on $[0,1]^{100}$ and 2) Spambase~\citep{Dua:2019} dataset containing feature vectors for $4601$ different emails for the purpose of spam detection.

\begin{figure}[p]
    \centering
    \begin{subfigure}[t]{\ifarxiv{0.48\columnwidth}\else{\columnwidth}\fi}
        \includegraphics[width=\columnwidth]{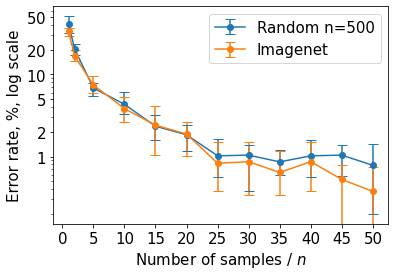}
        \caption{Prediction error depending on the number of samples for a random tree with $n\in \{100, 500\}$ nodes, and the ImageNet hierarchy.
        For each dataset, we average the results over 10 runs, and the error bars correspond to $10\%$ and $90\%$ quantiles.}
        \label{fig:scatter}
    \end{subfigure}
    \hfill
    \begin{subfigure}[t]{\ifarxiv{0.48\columnwidth}\else{\columnwidth}\fi}
        \includegraphics[width=\columnwidth]{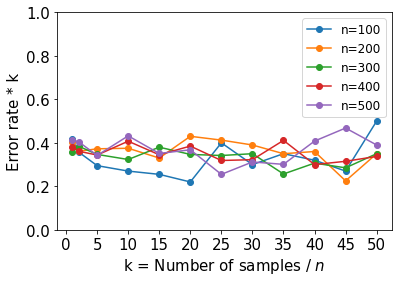}
        \caption{For random trees with $n$ nodes, we show dependence of $k \cdot \err$ on $k$, where $k$ as the ratio of number of samples to $n$.
        As Theorem~\ref{thm:natarajan_main} indicates, this number is close to a constant for different $n$ and $k$. Each data point shows the mean values over $10$ runs.}
        \label{fig:const}
    \end{subfigure}
    \caption{Binary Realizable case}
\end{figure}

%
\begin{figure}[p]
    \centering
    \ifarxiv
        \includegraphics[width=0.45\columnwidth]{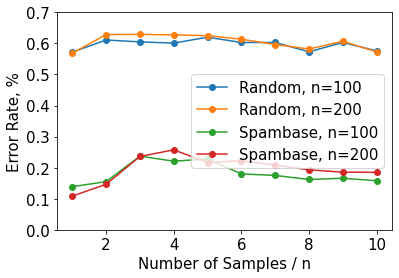}
    \else
        \includegraphics[width=0.65\columnwidth]{pics/spambase_and_random_identity.png}
    \fi
    \caption{Agnostic case. Prediction error depending on the number of samples. We sample $n$ random $100$-dimensional vectors and a subset of the Spambase dataset of size $n$.
    }
    \label{fig:spambase_random}
\end{figure}
\todo{Rearrange plots?}
\begin{figure}[p]
    \centering
    \includegraphics[width=\ifarxiv{0.4}\else{0.8}\fi\columnwidth]{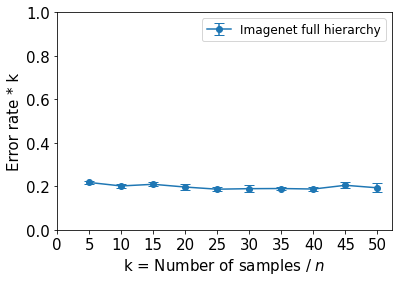}
    \caption{Non-binary Realizable case. For full ImageNet hierarchy, we show the dependence of $k \cdot \err$ on $k$, where $k$ is the ratio of samples to nodes ($n$).
    Similarly to Figure~\ref{fig:const}, this number is close to a constant for various $k$. Error bars correspond to $10\%$ and $90\%$ quantiles over $10$ runs.}
    \label{fig:imagenet_full}
\end{figure}

\paragraph{Binary Realizable case}

Given a ground-truth tree, we sample triplets from this tree uniformly at random.
Since the orientations for these triplets are not contradictory, we can construct a tree and make predictions according to the tree.
We first build a non-binary tree using the algorithm described in~\citet{ASSU81}, and then binarize it by replacing each non-binary node with a random binary tree on its children.

For this approach, Figure~\ref{fig:scatter} shows the dependence of the error rate on $k$, where $k$ is the ratio of the number of samples to the number of labels.
The results imply that the error rate depends on $k$ and is independent of the number of leaves or other properties of the ground-truth tree.
From Theorem~\ref{thm:natarajan_main} we know that the sample complexity is roughly proportional to $\frac{n}{\eps}$, and hence we expect $\eps \cdot k = \frac{\eps}{n} \cdot n_{\text{samples}}$ to be approximately constant.
Figure~\ref{fig:const} confirms this hypothesis since $\eps \cdot k$ is approximately $0.4$ for various values of $n$ and $k$.


\paragraph{Agnostic case}

For this scenario, we assume that the input is vectors in the Euclidean space, and we generate triplet constraints as follows.
Given a triplet of vectors $(a, b, c)$, we create constraint $\C abc$ if $\|a - b\| \le \min(\|a-c\|, \|b-c\|)$.
Such constraints can be contradictory if the dataset is not hierarchical, meaning that the setting is agnotic.
Similarly to the realizable case, Figure~\ref{fig:spambase_random} shows the dependence of the error rate on $k$, where $k$ is the ratio of the number of samples to the number of labels.
Since random vectors don't have any hierarchical structure, known samples don't provide sufficient information about the unseen samples, and hence the error rate is close to trivial $\nicefrac 23$ regardless of $k$.
On the other hand, since Spambase feature vectors have a hierarchical structure, they error rate on this dataset is significantly lower and slowly decreases with $k$.

\paragraph{Non-binary realizable case}

In this experiment on the full ImageNet hierarchy, constraints of type $\C abc$ are present, and hence we verify the theoretical result from Section~\ref{ssec:3_way_constraints}.
To handle these constraints, when partitioning a node using connected components, for each connected component we create a child of the node, i.e. we don't binarize the tree.

As shown in Figure~\ref{fig:imagenet_full}, similarly to the binary tree case, the product of the error rate and the number of constraints is approximately constant.
Note that the value in Figure~\ref{fig:imagenet_full} (0.2) is lower than in the value in Figure~\ref{fig:const} (0.4).
This is explained by the large number of $\TW ijk$ constraints in the hierarchy (since most of the nodes are separated at the top level of the hierarchy), and the fact that, when no samples are provided, all constraints are predicted as $\TW ijk$.

\section{Conclusion}

In this paper we give almost optimal bounds on the sample complexity of learning hierarchical tree representations of data from labeled tuples in both distributional (PAC-learning) and online case. Our experimental results confirm the convergence bounds predicted by the theory on trees generated from the ImageNet dataset.

\ifarxiv
    \bibliographystyle{plainnat}
\fi
\bibliography{bib}
\ifarxiv
    \newpage
    \appendix
    \section{Missing proofs from Section~\ref{sec:contr_orient}}
\label{app:missing_proofs}

Recall that we associate constraint $\C abc$ with edge $(a,b)$.
Intuitively, the tree is consistent with constraint, if it cuts off $c$ before cutting edge $\{a,b\}$.
\begin{definition}
    \label{def:generated_edges}
    Let $C$ denote the set of all possible constraints over $V$.
    We define the function $\genf\colon\ C \rightarrow V \times V$, referred to as the \emph{edge generating function}, as $\genf(\C abc) = (a,b)$.
    For any $C' \subseteq C$, we define $\genf(C') = \{\genf(t) \mid t \in C'\}$.
\end{definition}
Recall that we call the set of constraints contradictory if there is no tree consistent with these constraints.
We call a set $S$ closed w.r.t. the  set of triplets $\Delta$ if there exists an orientation $\orient{\Delta}$ of $\Delta$, such that $\genf(\orient{\Delta} |_{S})$ connects $S$.
\begin{lemma}
    \label{lem:closed_set_proof}
    A set of triplets $\Delta$ over $V$ allows a contradictory orientation iff there exists a set $S \subseteq V$ that is closed w.r.t. $\Delta$.
\end{lemma}

\begin{proof}
    $ $\newline
    $\Longrightarrow$: We will instead prove the contrapositive; i.e., we will assume that there is no set closed w.r.t. $\Delta$ and show that any orientation of $\Delta$ is non-contradictory.
    Given any orientation of $\orient{\Delta} = \bbr{\C {a_i}{b_i}{c_i}}_i$ of $\Delta$, we will build  in a top-down manner a hierarchical tree satisfying $\orient{\Delta}$. 
    
    At every step we have a set of points $S \subseteq V$ that we would like to split.
    In order to do so consider $\orient{\Delta} |_{S}$ and let $E_S = \genf(\orient{\Delta} |_S)$. By our earlier assumption, there does not exist a set that is closed w.r.t. $\Delta$ - in particular, $S$ is not such a set. Therefore, $E_S$ does not connect $S$ and there exists a disconnected bipartition $S = C_1 \cup C_2$. The algorithm cuts $C_1$ from $C_2$. 

    Clearly, the algorithm is well-defined (if there are multiple possible bipartitions, it selects an arbitrary one).
    Furthermore, the algorithm never breaks any constraints and, since it eventually splits all points, it must have satisfied all the constraints~-- thereby completing the proof.

    \vspace{3pt}
    \noindent $\Longleftarrow$: Let $S$ be a closed set w.r.t. $\Delta$. Then by Definition \ref{def:closed_set} there exists an orientation $\orient{\Delta}$ such that $\genf(\orient{\Delta}|_S)$ connects $S$.
    We next prove that $\orient{\Delta}$ is a contradictory orientation of $\Delta$.
    
    Assume towards contradiction that this is not the case.
    Therefore, there exists a hierarchical tree satisfying all constraints in $\orient{\Delta}$.
    Consider the first (closest to the root) tree node that cuts the set $S$.
    Since $\genf(\orient{\Delta}|_S)$ connects $S$, the cut must cut an edge in $\genf(\orient{\Delta}|_S)$.
    Denote this edge as $(a,b)$ and let the $\C abc$ be the constraint generating the edge.
    By the definition of $\orient{\Delta}|_S$, $c$ must belong to $S$, and hence constraint $\C abc$ is violated by the tree~-- in contradiction to our assumption, thereby concluding our proof.
\end{proof}

\begin{theorem}[Theorem~\ref{thm:crit_set_is_a_closed_set} restated]
    Let $\Delta$ be a set of triplets over $V$ of size $|\Delta| \geq |V| -1$.
    Then any critical set w.r.t. $\Delta$ is closed w.r.t. $\Delta$.
\end{theorem}
Before proving Theorem~\ref{thm:crit_set_is_a_closed_set}, we define the following crucial operation.
\begin{definition}[Reorientation]
    For a fixed triplet, when changing a constraint from $\C abc$ to another constraint (i.e. $\C bca$ or $\C acb$), we say that we \emph{reoriented} the constraint.
\end{definition}
\begin{proof}[Proof of Theorem~\ref{thm:crit_set_is_a_closed_set}]
    Let $S$ be a critical set and $\Delta |_S$ be the set of triplets induced by $S$. If $S$ is closed then the proof holds.
    Otherwise any orientation of $\Delta$ results in $\genf(\orient{\Delta}|_S)$ not connecting $S$.


    Consider $\orient{\Delta}$ such that the maximum forest $\mathcal F$ with vertices from $S$ and edges from $\genf(\orient{\Delta}|_S)$ has more than $1$ component.
    Among all such forests we consider the one with the maximum number of edges.
    Then, the forest $\mathcal F$ has at most $|S| - 2$ edges, and there exists a constraint $\C abc \in \orient{\Delta} |_S$ such that $\genf(\C abc)$ is not in the forest.
    Below, we refer to triplet $(a,b,c)$ as the \emph{unused triplet}.
    If two of these vertices, e.g. $a$ and $b$, belong to different trees, we can generate the edge $(a,b)$ to connect the trees, contradicting maximality of the forest $\mathcal F$.
    Hence, $a$, $b$ and $c$ must belong to the same tree, which we denote $T^*$.
    Among all $\orient{\Delta}$ and maximum forests $\mathcal F$ on $\orient{\Delta}$, we consider the ones minimizing the size of $T^*$.

    By minimality of $S$, the vertices of $T^*$ don't form a critical set, and hence there must exist a non-empty set of constraints $\orient{\Delta}^* = \bbr{\C {x_i}{y_i}{z_i}}_{i=1}^m$ such that $(x_i,y_i)$ is an edge in tree $T^*$ and $z_i$ belongs to a different tree.
    Let $U_1, \ldots, U_k$ be the connected components obtained by removing edges $\{(x_i,y_i)\}_{i=1}^m$ from $T^*$.
    It suffices to consider the following cases (other cases are symmetrical).

    \begin{figure}[t!]
        \centering
        \begin{subfigure}[t]{\ifarxiv{0.48\columnwidth}\else{\columnwidth}\fi}
            \includegraphics[width=\columnwidth]{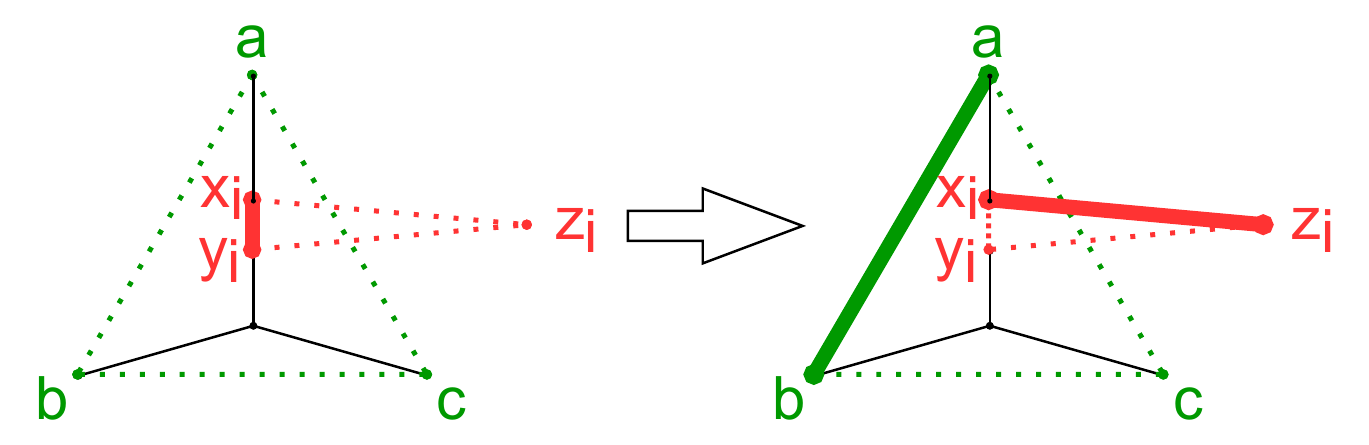}
            \caption{Case 1: $(x_i,y_i)$ separates $a$ from $b$ and $c$. In this case, we connect $T^*$ to the tree containing $z_i$}
            \label{fig:inside}
        \end{subfigure}
        \hfill
        \begin{subfigure}[t]{\ifarxiv{0.48\columnwidth}\else{\columnwidth}\fi}
            \includegraphics[width=\columnwidth]{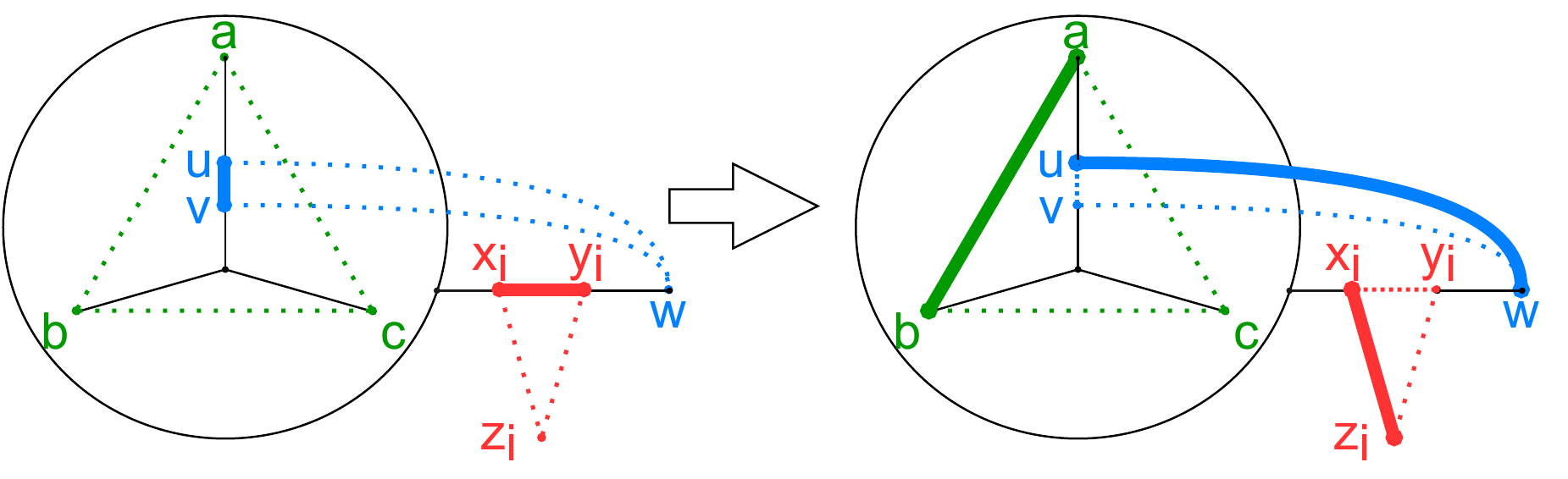}
            \caption{Case 2a: $(x_i,y_i)$ doesn't separate $a$, $b$ and $c$. If there exists an edge $(u,v)$ with the third vertex outside $U$ and separating $a$ from $b$ and $c$, we can connect $T^*$ to the tree containing $z_i$ by generating $(a,b)$ and reorient $(u,v)$}
            \label{fig:outside_inside}
        \end{subfigure}
        \begin{subfigure}[t]{\ifarxiv{0.48\columnwidth}\else{\columnwidth}\fi}
            \includegraphics[width=\columnwidth]{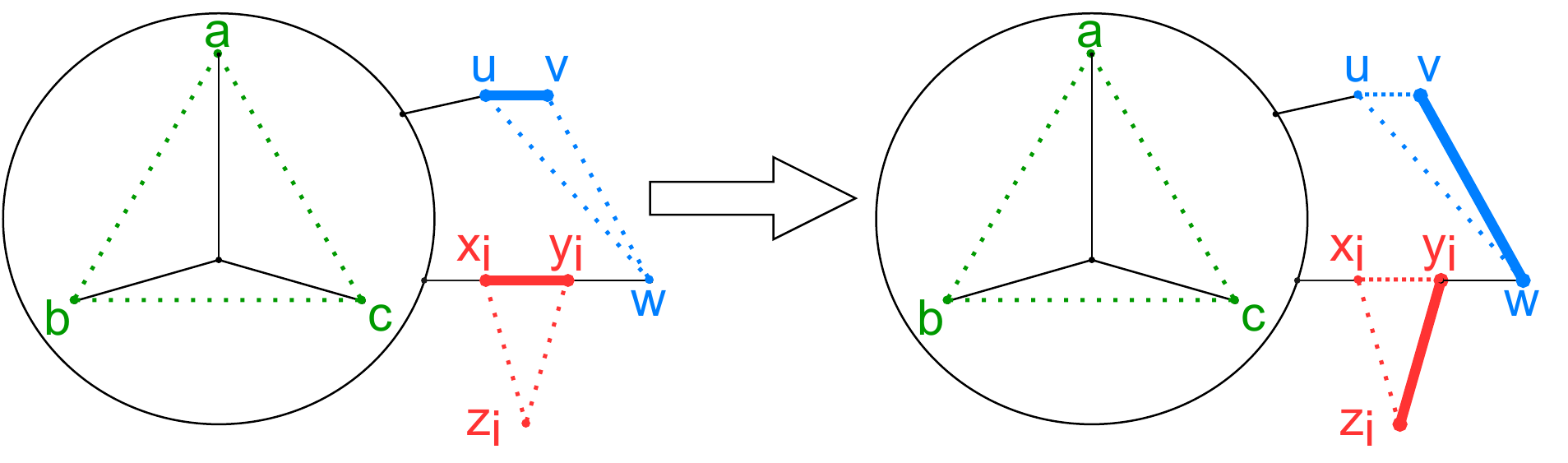}
            \captionsetup{singlelinecheck=off}
            \caption{Case 2b: $(x_i,y_i)$ doesn't separate $a$, $b$ and $c$.
            If there is a constraint $\C uvw$ such that
            \begin{compactenum}
                \item $u,v \in U$ and $w \notin U$,
                \item $(u, v)$ doesn't separate $a$, $b$ and $c$,
                \item $(u, v)$ doesn't separate $a$, $b$, $c$ from $x_i$, $y_i$ (unlike Case 2c in Figure~\ref{fig:outside_outside_sep}),
            \end{compactenum}
            then we reorient edges $(u, v)$ and $(x_i,y_i)$, connecting vertices $v$, $w$ and $y_i$ to the tree containing $z_i$, hence reducing the size of $T^*$}
            \label{fig:outside_outside}
        \end{subfigure}
        \hfill
        \begin{subfigure}[t]{\ifarxiv{0.48\columnwidth}\else{\columnwidth}\fi}
            \includegraphics[width=\columnwidth]{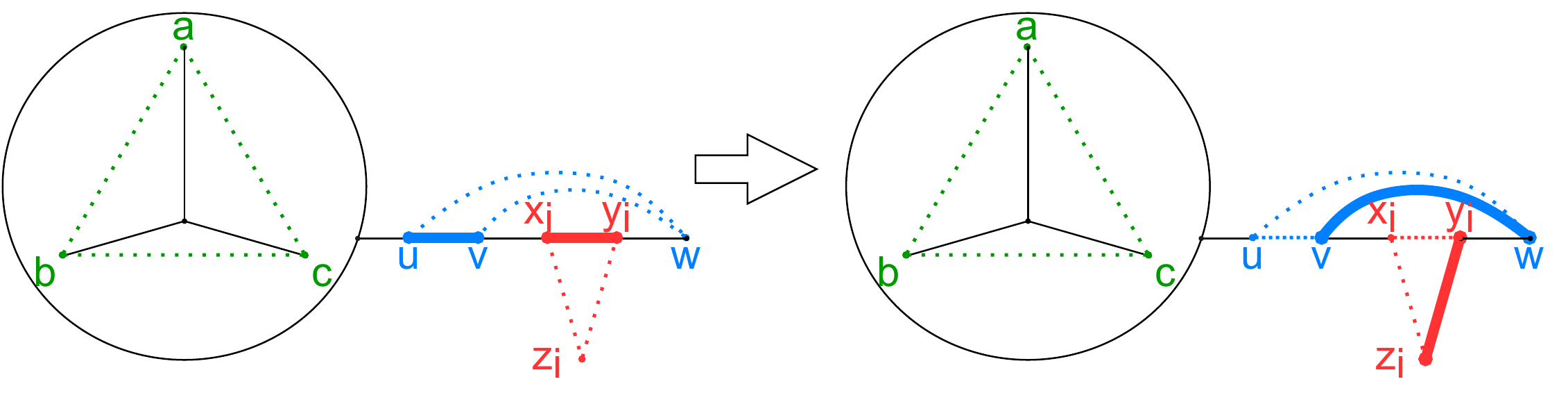}
            \captionsetup{singlelinecheck=off}
            \caption{Case 2c: $(x_i,y_i)$ doesn't separate $a$, $b$ and $c$. If there is a triplet $(u, v, w)$ such that
            \begin{compactenum}
                \item $u,v \in U$ and $w \notin U$,
                \item $(u, v)$ doesn't separate $a$, $b$ and $c$,
                \item $(u, v)$ separates $a$, $b$, $c$ from $x_i$, $y_i$ (unlike Case 2b in Figure~\ref{fig:outside_outside}),
            \end{compactenum}
            then we reorient edges $(u,v)$ and $(x_i, y_i)$, connecting vertices $v$, $w$, $x_i$ and $y_i$ to the tree containing $z_i$, hence reducing the size of $T^*$}
            \label{fig:outside_outside_sep}
        \end{subfigure}
        \caption{Case analysis in Theorem~\ref{thm:crit_set_is_a_closed_set}}
    \end{figure}
        
    \paragraph{Case 1: $a \in U_j$ and $b,c \notin U_j$ for some $j$}
    Then there exists a tree edge $(x_i,y_i)$ such that removing the edge separates $a$ from $b$ and $c$ (see Figure~\ref{fig:inside}).
    W.l.o.g. we assume that $x_i \in U_j$.
    Reorienting $(x_i,y_i)$ connects $U_j$ and the tree containing $z_i$, but also disconnects $U_j$ from the rest of $T^*$.
    By orienting $(a,b,c)$ as $\C abc$, we connect $U_j$ with the rest of $T^*$.
    To summarize, we connected $T^*$ to another tree, contradicting the maximality of forest $\mathcal F$.
    
    \paragraph{Case 2: $a,b,c \in U_j$ for some $j$}
    We denote $U = U_j$.
    Below we show that we can either connect $T^*$ to another tree (contradicting the maximality of the forest) or reduce the size of $T^*$ (contradicting the minimality of $T^*$) while still guaranteeing that the number of trees doesn't change and $T^*$ contains the unused triplet $(a,b,c)$.

    Since $U$ is not a critical set (since $S$ is a critical set, it is by definition minimal), there exists constraint $\C uvw$ such that $u,v \in U$ and $w \notin U$.
    Let $\bbr{\C {x_i}{y_i}{z_i}}_i$ be a set of constraints from $\orient{\Delta}$ such that for every $i$, edge $(x_i,y_i)$ is in $T^*$ and removing $(x_i, y_i)$ separates $w$ from $u$ and $v$.
    Let $\tilde T_i$ be the tree containing $z_i$.

    \paragraph{Case 2a: Removal of $(u, v)$ separates $a,b,c$ (Figure~\ref{fig:outside_inside})}
    There exists a tree edge $(u, v)$ such that $u,v \in U$, $w \notin U$, and w.l.o.g. $(u, v)$ separates $a$ from $b$ and $c$.
    Then reorienting $(u, v)$ and $(x_i,y_i)$ and generating edge $(a, b)$ connects $T^*$ to another tree, contradicting maximality of $\mathcal F$.

    \paragraph{Case 2b: Removal of $(u, v)$ does not separate $a,b,c$, and removal of $(u, v)$ doesn't separate $x_i,y_i$ from $a,b,c$ (Figure~\ref{fig:outside_outside})}
    We reorient $(v,u)$ to $(v,w)$, where $v$ is the node furthest from $a,b,c$) and $(x_i,y_i)$.
    Hence, the vertices $y_i, v, w$ (and other vertices connected to them) are disconnected from $T^*$ and connected to $\tilde T$.
    Hence, $T^*$ still contains the unused triplet $(a,b,c)$, and the size of the forest doesn't change; however, the size of $T^*$ decreases, contradicting the minimality of $T^*$.

    \paragraph{Case 2c: Removal of $(u, v)$ does not separate $a,b,c$, and removal of $(u, v)$ separates $x_i,y_i$ from $a,b,c$ (Figure~\ref{fig:outside_outside_sep})}
    We reorient $(u,v)$ to $(v,w)$, where $v$ is the node furthest from $a,b,c$) and $(x_i,y_i)$.
    Hence, the vertices $x_i, y_i, v, w$ are disconnected from $T^*$ and connected to $\tilde T$.
    Similar to Case 2b, we decrease the size of $T^*$, contradicting the minimality of $T^*$.

    In all cases, we show a contradiction, meaning that the maximal forest must have at least $|S|-1$ edges, implying that $S$ is connected.
\end{proof}

\subsection{\texorpdfstring{$k$}{k}-tuple constraints}

We first reduce the problem to the triplet case.
\begin{lemma}
    \label{lem:tuple_to_tiplet}
    Let $(a_1, a_2, \ldots, a_k)$ be a $k$-tuple.
    Then there exists a set of $k-2$ distinct triples $\Delta = \bbr{(a_i^{(t)}, a_j^{(t)}, a_\ell^{(t)})}_{t=1}^{k-2}$ such that for any orientation of $\Delta$ there exists a tree over $a_1, \ldots, a_k$ satisfying all constraints in the orientation.
\end{lemma}

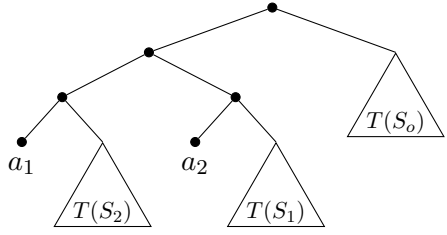
\begin{figure}[t!]
    \centering
    \
\newcommand{\treeparams}{for tree={s sep=1cm, l-=1.5em, solid node, math content}}
\bracketset{action character=@}

\begin{tikzext}
    \begin{forest}
        [ , @\treeparams,
          [
            [ 
              [ , label={below:$a_1$} ]
              [T(S_2), tria]
            ]
            [ 
              [ , label={below:$a_2$} ]
              [T(S_1), tria]
            ]
          ]
          [T(S_o), tria]
        ];
    \end{forest}
\end{tikzext}
    \caption{Lemma~\ref{lem:tuple_to_tiplet}: A tree satisfying a orientation of $\bbr{\C {a_1}{a_2}{a_t}}_{t=3}^{k}$ where $S_o$ is a set of labels $a_t$ satisfying $\C {a_1}{a_2}{a_t}$, $S_1$ is a set of points $a_t$ satisfying $C {a_2}{a_t}{a_1}$ and $S_2$ is a set of points $a_t$ satisfying $\C {a_1}{a_t}{a_2}$.
    $T(S_i)$ denotes an arbitrary tree on $S_i$}
    \label{fig:tuple_tree}
\end{figure}

\begin{figure*}[t!]
    \centering
    \newcommand{\treeparams}{for tree={s sep=0.7cm, l-=1.5em, solid node, math content}}
\bracketset{action character=@}

\begin{tikzext}
    \begin{forest}
        [ , @\treeparams,
          [
            [ 
              [ , label={below:$a$} ]
              [U_a, tria]
            ]
            [ 
              [ , label={below:$b$} ]
              [U_b, tria]
            ]
          ]
          [U_o, tria]
        ];
        \node[below=8em,align=center,anchor=center] {$T_i$};
    \end{forest}
    \begin{forest}
        [ , @\treeparams,
          [
            [ 
              [ , label={below:$a$} ]
              [V_a, tria]
            ]
            [ 
              [ , label={below:$b$} ]
              [V_b, tria]
            ]
          ]
          [V_o, tria]
        ];
        \node[below=8em,align=center,anchor=center] {$\orient{t_{i+1}}$};
    \end{forest}
    \begin{tikzpicture}
        \centering
        \draw[line width=12pt,my triangle,
            postaction={draw,white,line width=10pt,my triangle,shorten >=2pt,shorten <=1pt}](0,0) -- (1.5,0);
        \node at (0,-2) {};
    \end{tikzpicture}
    \begin{forest}
        [ , @\treeparams,
          [
            [ 
              [ , label={below:$a$} ]
              [ , s sep=0.1cm
                [U_a, tria]
                [V_a, tria]
              ]
            ]
            [ 
              [ , label={below:$b$} ]
              [ , s sep=0.1cm
                [U_b, tria]
                [V_b, tria]
              ]
            ]
          ]
          [ , s sep=0.1cm
            [U_o, tria]
            [V_o, tria]
          ]
        ];
        \node[below=9em,align=center,anchor=center] {$T_{i+1}$};
    \end{forest}
\end{tikzext}
    \caption{Theorem~\ref{thm:shatter_lb}: Given tree $T_i$ (satisfying $\orient{t_1}, \ldots, \orient{t_i}$), constraint $\orient{t_{i+1}}$, and $\{a,b\} = t_i \cap t_{i+1}$, we construct tree $T_{i+1}$ satisfying $\orient{t_1}, \ldots, \orient{t_{i+1}}$}
    \label{fig:tuple_lb}
\end{figure*}

\begin{proof}
    We define $\Delta = \bbr{(a_1, a_2, a_t)}_{t=3}^{k}$ and let $\orient{\Delta}$ be some orientation of $\Delta$.
    We partition $a_3, \ldots, a_k$ into three sets depending on which vertex is separated in $(a_1, a_2, a_t)$:
    \begin{compactitem}
        \item $a_t$ is in $S_1$ when $a_1$ is separated, i.e. $S_1 = \bbr{a_j \mid \C {a_2}{a_j}{a_1} \in \orient{\Delta}}$;
        \item $a_t$ is in $S_2$ when $a_2$ is separated, i.e. $S_2 = \bbr{a_j \mid \C {a_1}{a_j}{a_2} \in \orient{\Delta}}$;
        \item $a_t$ is in $S_o$ when $a_t$ is separated, i.e. $S_o = \bbr{a_j \mid \C {a_2}{a_2}{a_j} \in \orient{\Delta}}$.
    \end{compactitem}
    Then the tree shown in Figure~\ref{fig:tuple_tree}, where $T(S_i)$ is any tree over $S_i$, satisfies $\orient{\Delta}$.
\end{proof}

\begin{theorem}
    \label{thm:tuples}
    Let $|V|$ be a set and $\Delta$ be a set of $k$-tuples on $V$.
    If $|\Delta| \geq \left\lceil \frac{|V|-1}{k-2} \right\rceil$, there exists a contradictory orientation of $\Delta$.
\end{theorem}
\begin{proof}
    Let $\Delta = \{t_1, \ldots, t_{\ell}\}$ and for each $k$-tuple $t_i$ let $\Delta_i$ be the set of triplets from Lemma~\ref{lem:tuple_to_tiplet}.
    Let $M = \cup_i \Delta_i$ (w.l.o.g. we assume that $\Delta_i$ are disjoint since otherwise getting contradiction in the following reasoning is trivial).
    We first observe that if we can find a contradictory orientation of $M$, then by Lemma \ref{lem:tuple_to_tiplet} we can construct a contradictory orientation of $\Delta$.
    Namely, if $\orient{\Delta}_1, \ldots, \orient{\Delta}_\ell$ is a contradictory orientation of $M$, then for each $i$ we select a tree which is consistent with $\orient{\Delta}_i$.
    Therefore, it is enough to show that there exists a contradictory orientation of $M$.
    
    We have $|M| = |\Delta| \cdot (k-2)$ and since $|\Delta| \geq \left\lceil \frac{|V|-1}{k-2} \right\rceil$, we are guaranteed that $|M| \geq |V| - 1$.
    Thus, by Theorem \ref{thm:triplets} there exists a contradictory orientation of $\Delta$, thereby concluding the proof.
\end{proof}

\paragraph{Note} For $|\Delta|=1$, we have $|V| - 1 \le k-2$, i.e. $|V| < k$, and hence no $k$-tuple can be constructed.
Allow repeating elements in a tuple trivializes the problem: a contradiction can be achieved by using a constraint of form $\C aba$.

Next, we observe that one cannot achieve a better bound than Theorem~\ref{thm:tuples}, i.e. there exists a set of $k$-tuples of size $\left\lceil \frac{|V|-1}{k-2} \right\rceil - 1$ which doesn't have a contradictory orientation.

\begin{theorem}
    \label{thm:shatter_lb}
        There exists a set $\Delta$ of $k$-tuples with $|\Delta| = \left\lceil \frac{|V|-1}{k-2} \right\rceil - 1$, which does not have a contradictory orientation.
\end{theorem}

\begin{proof}
    Given set $V = \{a_1, \ldots, a_{|V|}\}$, we select $\Delta$ as:
    \begin{align*}
        \Delta = \Big\{
            &(a_1, a_2, \ldots, a_k), \\
            &(a_{k-1}, \ldots, a_{2k - 2}), \\
            &(a_{2k - 1}, \ldots, a_{3k - 4}), \\
            &\quad\quad\ldots \\
            &(a_{(|\Delta|-1)(k-2) + 1}, \ldots, a_{|\Delta| (k-2) + 2}) \Big\},
    \end{align*}
    i.e. $\Delta = \{t_1, \ldots, t_m\}$ where $|t_i|=k$, $|t_i \cap t_{i+1}| = 2$, and $t_i$ and $t_j$ are disjoint for $|i - j| > 1$.
    Note that $\Delta$ uses at most $|V|$ points, since
    \begin{flalign*}
        |\Delta| (k-2) + 2
        &= \left(\left\lceil \frac{|V|-1}{k-2} \right\rceil - 1\right) (k - 2) + 2 \notarxiv{\\&}
        \le \frac{|V| - 1 - 1}{k-2} (k - 2) + 2
        \le |V|
    .\end{flalign*}
    For any orientation $\orient{\Delta} = \{\orient{t}_1, \ldots, \orient{t}_{|\Delta|}\}$ of $\Delta$, we now construct a tree satisfying $\orient{\Delta}$.
    We construct a sequence of trees $T_1, \ldots, T_{|\Delta|}$ such that $T_i$ satisfies all of $\orient{t_1}, \ldots, \orient{t_i}$.
    We define $T_1 = \orient{t_1}$.
    Then, given $T_i$ and $\orient{t}_{i+1}$, we construct $T_{i+1}$ as shown in Figure~\ref{fig:tuple_lb}.
    Since all constraints satisfied by $T_i$ and $\orient{t_{i+1}}$ are satisfied by $T_{i+1}$, then $T_{i+1}$ satisfies $\orient{t_1}, \ldots, \orient{t_{i+1}}$.
\end{proof}

\subsection{Constraints of form \texorpdfstring{$\TW abc$}{[a|b|c]}.}
\label{ssec:3_way_constraints}

In this section, we extend the result from Theorem~\ref{thm:natarajan} to the case when constraints of form $(a|b|c)$ are allowed, i.e. tree $T$ satisfies the constraint if $\lca_T(a,b)=\lca_T(b,c)=\lca_T(a,c)$.
We show exactly the same bound as in Theorem~\ref{thm:natarajan}, which allows us to use other results without changes.

Similarly to Section~\ref{sec:contr_orient}, for an orientation $\overrightarrow \Delta$ and set $S \subseteq V$ we define $E_S = \bbr{(a,b) \mid \C abc \in \overrightarrow \Delta|_S}$.
Additionally, we introduce the following operation to handle $\TW abc$ constraints.

\begin{definition}
    Let $[a|b|c] \in \overrightarrow \Delta|_S$ and let $\mathcal E$ be a set of edges. If at least two of the vertices $a,b,c$ are connected by a path in $\mathcal E$, we say that $\mathcal E \cup \{(a,c), (a,b), (b,c)\}$ is a $t$-extension of $\mathcal E$ using $\TW abc$.
\end{definition}

\textbf{Note:} Importantly, the definition requires that two endpoints of the constraint are already connected by existing edges.

Intuitively, if $a$ and $b$ are connected, we can't cut them at the current level.
Hence, we also can't cut $c$ from $a$ and $b$ due to the constraint.
We hence connect $c$ to $a$ and $b$ to show that they can't be cut at the current level.

\begin{definition}
Let $\Delta$ be a set of triplets.
$S$ is \textbf{closed} if there exists an orientation $\orient{\Delta}$ and a sequence $\mathcal E_0, \ldots, \mathcal E_k$, where: 
\begin{itemize}
    \item $\mathcal E_0 = E_S$,
    \item $\mathcal E_{i+1}$ is a $t$-extension of $\mathcal E_i$ for all $i$ using a constraint from $\overrightarrow \Delta|_S$,
    \item $\mathcal E_k$ connects $S$.
\end{itemize}
\end{definition}

\begin{lemma}
    \label{lem:nonbinary:closed_is_contradictory}
    Let $\Delta$ be a set of triplets on $V$.
    Then there exists a contradictory orientation of $\Delta$ iff there exists a closed subset of $V$.
\end{lemma}
\begin{proof}
    $ $\\
    \noindent $\Longleftarrow$: Let $S$ be a closed set. For contradiction, assume that the hierarchical tree exists and consider the first tree node cutting $S$ into $\mathcal U = \{U_1, \ldots, U_\ell\}$. The tree can't cut an edge from $E_S$ (see Lemma~\ref{lem:closed_set_proof}), so $U_i$ are not connected by edges from $E_S$.
    
    Since $S$ is critical, there exists $\orient{\Delta}$ and a sequence $\mathcal E_0, \ldots, \mathcal E_k$ such that $\mathcal E_0 = E_S$, $\mathcal E_k$ connects $S$, and $\mathcal E_i$ is a $t$-extensions of $\mathcal E_{i-1}$ using constraint $\TW{a_i}{b_i}{c_i} \in \overrightarrow \Delta|_S$. Let $\mathcal E_{i^*}$ be the first $t$-extension in the sequence connecting some of the components from $\mathcal U$ (such $i^*$ exists since $E_k$ connects $S$).
    
    Since $\mathcal E_{i^*}$ is a $t$-extension of $\mathcal E_{i^*-1}$ using $\TW{a_{i^*}}{b_{i^*}}{c_{i^*}}$, there exists a path $P$ in $\mathcal E_{i^*-1}$ connecting (w.l.o.g.) $a_{i^*}$ and $b_{i^*}$. This path consists of the following types of edges:
    \begin{itemize}
        \item Edges from $E_S$. All endpoints of such an edge must belong to the same component from $\mathcal U$.
        \item $(a_i,c_i)$, where $(a_i,c_i) = \mathcal E_{i} \setminus \mathcal E_{i-1}$ for $i < i^*$. By our assumption, $\mathcal E_{i^*}$ is the first $t$-extension connecting different components from $\mathcal U$, and hence $(a_i,c_i)$ belong to the same component from $\mathcal U$.
    \end{itemize}
    Hence, all nodes in $P$ lie in the same $U^* \in \mathcal U$, implying $a_{i^*},b_{i^*} \in U^*$.
    Since $c_{i^*} \notin U^*$, this leads to a contradiction since all of $a$, $b$ and $c$ must belong to the same component or to different components.

    \vspace{1em}
    \noindent $\Longrightarrow$: We instead prove the contrapositive: if there is no critical set, then for any orientation $\orient{\Delta}$ we can build a hierarchical tree.
    Let the current tree node correspond to the set $S$ of size at least $2$. Since $S$ is not critical, it's not connected after performing all possible $t$-extensions.
    Let $\mathcal U = \{U_1, \ldots, U_k\}$ be the connected components after performing all $t$-extensions (note that any maximal sequence of non-trivial contractions results in the same connected components).
    We show that splitting $S$ into $\mathcal U$ doesn't violate any constraint.
    
    First, it doesn't violate any constraints of form $\C abc$ since it doesn't cut an edge from $E_S$.
    It remains to show that it doesn't violate any constraint of form $\TW abc$.
    If $\TW abc$ can be used for $t$-extension, then $a$, $b$ and $c$ must be connected (by definition of $t$-extension and the fact that we performed all possible $t$-extensions) and belong to the same component; hence, their constraints are not violated.
    It remains to consider constraints $\C abc$ that can't be used for $t$-extension.
    But for such constraints, none of its pair of endpoints is connected, and hence $a$, $b$ and $c$ belong to different components.
    Hence, no constraints are violated.
\end{proof}

    \section{Missing proofs from Section~\ref{sec:pac}}
\label{app:n_dimension}

\begin{theorem}[Theorem~\ref{thm:natarajan} restated]
    For any $V$, we have $\ndim(H_3(V)) = |V|-2$.
\end{theorem}
\begin{proof}
    \noindent $\ndim(H_3(V)) \ge |V|-2$: follows from Lemma~\ref{lem:natarajan_lb}.

    \vspace{1em}
    \noindent $\ndim(H_3(V)) \le |V|-2$: For contradiction, assume that there exists an N-shattered set of triplets $\Delta = \{t_i\}_{i=1}^m$ of size $m \ge |V|-1$.
    Let $\orient{t_i}^{(1)}$ and $\orient{t_i}^{(2)}$ be the orientations from the definition of N-shattering for each $t_i$.
    By Theorem~\ref{thm:triplets}, there exists a contradictory orientation $\orient{\Delta}^* = \{\orient{t_i}^*\}_{i=1}^k$.
    If for all $i$ we have $\orient{t_i}^* = \orient{t_i}^{(1)}$ or $\orient{t_i}^* = \orient{t_i}^{(2)}$, then we have a contradiction: the set is not N-shattered since there exists $f:\ [k] \to [2]$ such that the orientation $\bbr{\orient{t_i}^{(f(i))}}_{i=1}^m = \orient{\Delta}^*$ is contradictory.
    
    It remains to consider the case when for some $t_i \in \Delta$ we have $\orient{t_i}^* \ne \orient{t_i}^{(1)}$ and $\orient{t_i}^* \ne \orient{t_i}^{(2)}$.
    Among all contradictory orientations we consider $\orient{\Delta}^*$ which has the smallest number of such $t_i$.
    Let $S$ be the critical set.
    We denote $\orient{t_i}^* = \C abc$, $\orient{t_i}^{(1)} = \C acb$ and $\orient{t_i}^{(2)} = \C bca$.
    Constraint $\C abc$ generates edge $(a,b)$.
    If $(a,b,c)$ doesn't belong to $\Delta|_S$ or removing edge $(a,b)$ doesn't disconnect $S$, then the orientation of this edge is not important and we can change the triplet orientation arbitrarily to $\orient{t_i}^{(1)}$ or $\orient{t_i}^{(2)}$.
    
    Otherwise, removing this edge partitions $S$ into two connected components $U$ and $V$ such that $a \in U$ and $b \in V$.
    However, we can restore connectivity using $\orient{t_i}^{(1)}$ or $\orient{t_i}^{(2)}$:
    if $c \in U$, then $b$ and $c$ are in different connected components, and we can connect $S$ using edge generated by $\C bca$.
    Similarly, if $c \in V$, then $a$ and $c$ are in different connected components, and we connect $S$ using $\C acb$.
    
    In both cases, we reorient $t_i$ according to either $\orient{t_i}^{(1)}$ or $\orient{t_i}^{(2)}$ while maintaining connectivity of $S$ and not changing orientations of other triplets.
    Hence, we have contradiction with the assumption that $\orient{t_i}^*$ has smallest number of $t_i$ such that $\orient{t_i}^* \ne \orient{t_i}^{(1)}$ and $\orient{t_i}^* \ne \orient{t_i}^{(2)}$
    Hence such $t_i$ doesn't exist and $\Delta$ is not N-shattered.
\end{proof}

\begin{theorem}
    \label{thm:natarajan_nonbinary}
    Let $H_3^*$ be defined identically to $H_3$ (Definition~\ref{def:hypothesis}) with the only change that the ground-truth trees don't have to be binary.
    Then $\ndim(H_3^*(V)) = |V|-2$.
\end{theorem}

\begin{proof}
    $ $
    \paragraph{$\ndim(H_3^*(V)) \ge |V|-2$:} since the possible labels for $H_3$ is a subset of labels for $H_3^*$, the lower bound on $\ndim(H_3^*)$ follows from the lower bound on $\ndim(H_3)$.
    
    \paragraph{$\ndim(H_3^*(V)) \le |V|-2$:} To simplify the presentation, w.l.o.g. we assume that $V$ is the critical set.
    From Theorem~\ref{thm:crit_set_is_a_closed_set}, we know that for any set of triplets $\Delta$ with $|\Delta| \ge |V| - 1$ there exists a contradictory orientation $\orient{\Delta}$ which uses only constraints of form $\C abc$. 
    
    Towards contradiction, let's assume that $\Delta$ can be N-shattered using $H_3^*$.
    For every triplet $t_i = (a_i, b_i, c_i)$ from $\Delta$, let $\orient{t_i}^{(1)}$ and $\orient{t_i}^{(2)}$ be allowed orientations from the definition of N-shattering.
    We show that we can reorient every constraint according to $o^{(1)}_i$ or $o^{(2)}_i$ while maintaining the connectivity of $V$.

    Let $\orient{\Delta} = \bbr{\orient{t_1}, \ldots, \orient{t}_{|V|-1}} = \bbr{\C{a_i}{b_i}{c_i}}_{i=1}^{|V|-1}$ be the contradictory orientation of $\Delta$ according to Theorem~\ref{thm:crit_set_is_a_closed_set}.
    We process the constraints depending on types of $\orient{t_i}^{(1)}$ and $\orient{t_i}^{(2)}$ in the following order:\todo{Finish}
    \begin{enumerate}
        \item Constraints for which $\orient{t_i} = \orient{t_i}^{(1)}$ or $\orient{t_i} = \orient{t_i}^{(2)}$~-- these constraints use labels from the definition of N-shattering, and hence we don't reorient them.
        \item Constraints $\orient{t_i} = \C{a_i}{b_i}{c_i}$ for which $\orient{t_i}^{(1)} = \C{a_i}{c_i}{b_i}$ and $\orient{t_i}^{(2)} = \C{b_i}{c_i}{a_i}$.
        We reorient all such constraints without breaking connectivity, see Theorem~\ref{thm:natarajan}.
        \item Constraints for which $\orient{t_i}^{(1)} = \C{a_i}{c_i}{b_i}$ and $\orient{t_i}^{(2)} = \TW{a_i}{b_i}{c_i}$.
        We say that such constraints are of Type 3 and perform case analysis on such constraints:
    \end{enumerate}

    \textbf{Case 1:} If there exists a constraint such that reorienting it using $\orient{t_i}^{(1)}$ doesn't break connectivity, then we reorient the triplet using $\orient{t_i}^{(1)}$.
    
    In the remaining cases, we assume that using $\orient{t_i}^{(1)}$ breaks connectivity for all $i$, i.e. for each Type 3 constraint $\C{a_i}{b_i}{c_i}$, reorienting it as $\C{a_i}{c_i}{b_i}$ breaks connectivity.
    In particular, it means that $a_i$ and $c_i$ are in the same component after removing edge $(a_i,b_i)$.
    
    \textbf{Case 2} (Figure~\ref{fig:type3_no_between}): There exist constraint $\C{a_{i^*}}{b_{i^*}}{c_{i^*}}$ such that there is no edge from other Type 3 constraint on the path from $a_{i^*}$ to $c_{i^*}$.
    Then we reorient the triplet as $\orient{t_{i^*}}^{(2)}$, i.e. $\TW{a_{i^*}}{b_{i^*}}{c_{i^*}}$.
    Note that $a_{i^*}$ and $c_{i^*}$ are connected, and, since there is no edge from other Type 3 constraints between $a_{i^*}$ and $c_{i^*}$, they will always remain connected, and hence we can perform $t$-extension using $\TW{a_{i^*}}{b_{i^*}}{c_{i^*}}$.
    Again, after the $t$-extension, $V$ remains connected.
    \begin{figure*}[t!]
        \centering
        \begin{subfigure}[t]{\ifarxiv{\columnwidth}\else{\columnwidth}\fi}
            \centering
\begin{tikzext}
\begin{tikzpicture}[every node/.style={inner sep=2pt}]
    \node (a) at (0,0) {$a_{i^*}$};
    \node (b) at (-0.95, -0.5) {$b_{i^*}$};
    \node (c) at (0, -1) {$c_{i^*}$};


    \draw[blue, decorate, decoration={snake, amplitude=0.5mm}] ($(c) + (0.2,0)$) arc(-90:90:2.15 and 0.5);

    \draw[red, thick] (a) -- (b);
    \draw[dashed, red, thin] (a) -- (c) -- (b);

    \draw[line width=12pt,my triangle,
        postaction={draw,white,line width=10pt,my triangle,shorten >=2pt,shorten <=1pt}](3,-0.5) -- (5,-0.5);
    
    \node (a2) at (6.8,0) {$a_{i^*}$};
    \node (b2) at (6.8 - 0.95, -0.5) {$b_{i^*}$};
    \node (c2) at (6.8, -1) {$c_{i^*}$};

    \draw [rounded corners=10pt] ($(a2)+(0.3,0.5)$) -- ($(c2)+(0.3,-0.5)$) -- ($(b2)+(-0.5,0)$) -- cycle;

     \draw[blue, decorate, decoration={snake, amplitude=0.5mm}] ($(c2) + (0.2,0)$) arc(-90:90:2.15 and 0.5);
     \draw[ForestGreen, thick] (a2) -- (c2) -- (b2) -- (a2);
\end{tikzpicture}
\end{tikzext}
            \caption{Case 2: for a Type 3 constraint $\C{a_{i^*}}{b_{i^*}}{c_{i^*}}$, if there is no edge from other Type 3 constraint on the path from ${a_{i^*}}$ to ${c_{i^*}}$, then we reorient the constraint as $\TW{a_{i^*}}{b_{i^*}}{c_{i^*}}$}
            \label{fig:type3_no_between}
        \end{subfigure}

        \vspace{1em}
        \begin{subfigure}[t]{\ifarxiv{\columnwidth}\else{\columnwidth}\fi}
            \centering
            \newcommand{\typeTreeEdgeLength}{1.2cm}
\begin{tikzpicture}[every node/.style={inner sep=1pt}]
    \node (o) at (0,0) {};

    \node (a1) at ({cos(30) * \typeTreeEdgeLength}, {sin(30) * \typeTreeEdgeLength}) {$a_1$};
    \node (b1) at ({cos(30) * 2 * \typeTreeEdgeLength}, {sin(30) * 2 * \typeTreeEdgeLength}) {$b_1$};
    \node (c3) at ({cos(30) * 3 * \typeTreeEdgeLength}, {sin(30) * 3 * \typeTreeEdgeLength}) {$c_3$};
    \draw[thick] (o) -- (a1);
    \draw[red, thick] (a1) -- (b1);
    \draw[thick] (b1) -- (c3);

    \node (a2) at ({cos(-90) * \typeTreeEdgeLength}, {sin(-90) * \typeTreeEdgeLength}) {$a_2$};
    \node (b2) at ({cos(-90) * 2 * \typeTreeEdgeLength}, {sin(-90) * 2 * \typeTreeEdgeLength}) {$b_2$};
    \node (c1) at ({cos(-90) * 3 * \typeTreeEdgeLength}, {sin(-90) * 3 * \typeTreeEdgeLength}) {$c_1$};
    \draw[thick] (o) -- (a2);
    \draw[red, thick] (a2) -- (b2);
    \draw[thick] (b2) -- (c1);
    
    \node (a3) at ({cos(150) * \typeTreeEdgeLength}, {sin(150) * \typeTreeEdgeLength}) {$a_3$};
    \node (b3) at ({cos(150) * 2 * \typeTreeEdgeLength}, {sin(150) * 2 * \typeTreeEdgeLength}) {$b_3$};
    \node (c2) at ({cos(150) * 3 * \typeTreeEdgeLength}, {sin(150) * 3 * \typeTreeEdgeLength}) {$c_2$};
    \draw[thick] (o) -- (a3);
    \draw[red, thick] (a3) -- (b3);
    \draw[thick] (b3) -- (c2);

    \draw[dashed, red, thin] (a1) -- (c1) -- (b1);
    \draw[dashed, red, thin] (a2) -- (c2) -- (b2);
    \draw[dashed, red, thin] (a3) -- (c3) -- (b3);
    
\end{tikzpicture}
\begin{tikzpicture}
    \draw[line width=12pt,my triangle,
        postaction={draw,white,line width=10pt,my triangle,shorten >=2pt,shorten <=1pt}](0,0) -- (2,0);    
    \node at (0, -3) {};
\end{tikzpicture}
\begin{tikzpicture}[every node/.style={inner sep=0pt}]
    \node (o) at (0,0) {};

    \node (a1) at ({cos(30) * \typeTreeEdgeLength}, {sin(30) * \typeTreeEdgeLength}) {$a_1$};
    \node (b1) at ({cos(30) * 2 * \typeTreeEdgeLength}, {sin(30) * 2 * \typeTreeEdgeLength}) {$b_1$};
    \node (c3) at ({cos(30) * 3 * \typeTreeEdgeLength}, {sin(30) * 3 * \typeTreeEdgeLength}) {$c_3$};
    \draw[thick] (o) -- (a1);
    \draw[thick] (b1) -- (c3);

    \node (a2) at ({cos(-90) * \typeTreeEdgeLength}, {sin(-90) * \typeTreeEdgeLength}) {$a_2$};
    \node (b2) at ({cos(-90) * 2 * \typeTreeEdgeLength}, {sin(-90) * 2 * \typeTreeEdgeLength}) {$b_2$};
    \node (c1) at ({cos(-90) * 3 * \typeTreeEdgeLength}, {sin(-90) * 3 * \typeTreeEdgeLength}) {$c_1$};
    \draw[thick] (o) -- (a2);
    \draw[thick] (b2) -- (c1);
    
    \node (a3) at ({cos(150) * \typeTreeEdgeLength}, {sin(150) * \typeTreeEdgeLength}) {$a_3$};
    \node (b3) at ({cos(150) * 2 * \typeTreeEdgeLength}, {sin(150) * 2 * \typeTreeEdgeLength}) {$b_3$};
    \node (c2) at ({cos(150) * 3 * \typeTreeEdgeLength}, {sin(150) * 3 * \typeTreeEdgeLength}) {$c_2$};
    \draw[thick] (o) -- (a3);
    \draw[thick] (b3) -- (c2);

    \draw[green, thick] (a1) -- (c1);
    \draw[green, thick] (a2) -- (c2);
    \draw[green, thick] (a3) -- (c3);
    
\end{tikzpicture}
            \caption{Case 3: when we have constraints $\C{a_1}{b_1}{c_1}, \ldots, \C{a_\ell}{b_\ell}{c_\ell}$ such that edge $(a_{i+1}, b_{i+1})$ separates $a_i$ from $c_i$ (we equalize $\ell+1$ with $1$), we can reorient all constraints as $\C{a_i}{c_i}{b_i}$ while preserving connectivity.
            Note that if, for example, $a_1$ and $b_1$ are swapped, then it would fit under Case 1: reorienting $\C{a_1}{b_1}{c_1}$ as $\C{a_1}{c_1}{b_1}$ alone would preserve connectivity.
            }
            \label{fig:type3_between}
        \end{subfigure}
        \caption{Case analysis of Type 3 constraints in Theorem~\ref{thm:natarajan_nonbinary}}
    \end{figure*}
    
    \textbf{Case 3} (Figure~\ref{fig:type3_between}): For each constraint $\C {a_i}{b_i}{c_i}$, there exists another constraint $\C{a_j}{b_j}{c_j}$ such that edge $(a_j,b_j)$ is on the path from $a_i$ to $c_i$. Consider a graph on these constraints, where there exists an edge $i \to j$ when $(a_j,b_j)$ is on the path from $a_i$ to $c_i$. Since Case 2 is not realized, the graph has a cycle $i_1 \to i_2 \to \ldots \to i_\ell \to i_1$.
    Among such cycles, we consider the cycle with the smallest length.
    To simplify the notation, we assume that the cycle is exactly $1 \to 2 \to \ldots \to \ell \to 1$ and equalize $\ell+1$ with $1$.
    
    The important observation is that, for every $i$, among the edges in the cycle, only $(a_{i + 1}, b_{i+1})$ cuts $a_i$ from $c_i$.
    Otherwise, if there exists another $j$ with $(a_j, b_j)$ cutting $a_{i}$ from $c_{i}$, we can remove $i+1, \ldots, j-1$ from the cycle, hence reducing its size and contradicting its minimality.
    Hence, one of the following holds regardless of the orientation of the triplets in the cycle (importantly, since Case 1 is not realized, $a_i$ and $c_i$ are in the same connected component after removing edge $(a_i,b_i)$):
    \begin{itemize}
        \item $a_{i}$ is connected to $a_{i+1}$ and $c_{i}$ is connected to $b_{i+1}$.
        \item $a_{i}$ is connected to $b_{i+1}$ and $c_{i}$ is connected to $a_{i+1}$.
    \end{itemize}
    We reorient the constraints in the cycle using the corresponding $\orient{t_{i}}^{(1)}$, i.e. $\C{a_{i}}{b_{i} }{c_{i}} \to \C{a_{i}}{c_{i}}{b_{i}}$ for all $s$.
    Since after the reorientation $a_{i}$ is connected to $c_{i}$, in both cases, $a_{i}$ is connected to both $a_{i+1}$ and $b_{i+1}$.
    Hence, all of $a_{1}, \ldots, a_{\ell}$ and $b_{1}, \ldots, b_{\ell}$ are connected after the reorientation, and hence $V$ remains connected.
    
    As long as a Type 3 constraint exists, we can apply one of Cases 1, 2 or 3, reducing the number of such constraints.
    Hence, after at most $|V|$ reorientations, all constraints are oriented in accordance to the definition of N-shattering, concluding the proof.
\end{proof}
\todo{Pictures}
From the above theorem, the main result follows.
\begin{theorem}
    \label{thm:non-binary}
    For a constant $k$, the sample complexity of learning non-binary hierarchically labeled $k$-tuples, denoted by $m_{H_k^*}^r(\eps, \delta)$ in the realizable setting and $m_{H_k^*}^a(\eps, \delta)$ in the agnostic setting, is bounded by:
    \begin{align*}
        C_1 \frac{n + \log \frac 1\delta}{\eps} \le m_{H_k^*}^r(\eps, \delta) & \le C_2 \frac{n \log \frac 1\eps + \log \frac 1\delta}{\eps} \\
        C_1 \frac{n + \log \frac 1\delta}{\eps^2} \le m_{H_k^*}^a(\eps, \delta) & \le C_2 \frac{n + \log \frac 1\delta}{\eps^2}
    \end{align*}
\end{theorem}


\section{Littlestone Dimension and Online Learning}
\label{sec:littlestone}

In this section, we show tight learning bounds of hierarchically labeled $k$-tuples for any constant integer $k > 0$. These bounds also hold for non-binary hierarchical trees.

Throughout this section, we consider an online learning setting as in~\citet{DanielySBS15}, defined as follows: in a series of rounds $\tau=1,\dots,T$, an online algorithm $\mathcal{A}$ receives in each round a $k$-tuple $t_\tau \in V^k$, and outputs an orientation $\orient{t}_\tau$ of $t_\tau$.
At the end of a round, a ``correct'' output $\orient{t}_\tau^*$ is revealed to the algorithm. The choice of $\orient{t}_\tau$ may only depend on $t_1,\dots,t_{\tau-1}$ and $\orient{t}_1^*,\dots,\orient{t}_{\tau-1}^*$. Our goal is to minimize the number of mistakes the algorithm makes $\err_T = \left|\left\{\tau \in {1,\dots,T} \mid \orient{t}_\tau \neq \orient{t}_\tau^*\right\}\right|$.

We say that a sequence is realizable if there is a hierarchical tree on $V$ satisfying all constraints $\orient{t}_1^*,\dots,\orient{t}_T^*$. In the agnostic (i.e. non-realizable setting), we say that the algorithm has regret $K$ if the number of mistakes it makes is $\mathrm{OPT}+K$ where $\mathrm{OPT}$ the minimum number of mistakes achieved by any tree on the sequence.

We give a tight $\Theta(n\log{n})$ bound on the multi-class Littlestone dimension, which was introduced by~\citet{DanielySBS15} as an extension of the binary-class Littlestone dimension~\citep{Littlestone87}.
Similarly to the Natarajan dimension, the notion of (multi-class) Littlestone dimension almost tightly characterizes the mistake and regret bounds of a hypothesis class in the online model, in both the realizable and agnostic setting.

\begin{definition}[Littlestone dimension]\label{def:littlestone}
    Let $V$ be an $n$ point set, and $k$ be some constant parameter.
    Let $L$ be a complete binary rooted tree such that each of its internal nodes is labeled by a $k$-tuple $t$ and each of its two edges to its children is labeled by a different orientation on $t$.
    
    We say that $L$ is Littlestone-shattered if for each path from the root to a leaf which traverses the nodes $t_1, \ldots, t_l$, there is a hierarchical tree on point set $V$ such that for each $\tau$, the label of the edge $(t_\tau,t_{\tau+1})$ is a constraint satisfied by the tree.
    The Littlestone dimension $\ldim$ of learning a hierarchical tree representation is the maximum depth of a full binary tree that is Littlestone-shattered (as a function of $n$).
\end{definition}
The connection between the Littlestone dimension and online learning is described in Theorem~\ref{thm:littlestone_to_sample}. We refer to~\citet[Section 5.1]{DanielySBS15} for a full survey of the multi-class Littlestone dimension.
\begin{theorem}[\citet{DanielySBS15}, Theorems 24, 25, 26, informal]
    \label{thm:littlestone_to_sample}
    In the realizable online setting, there exists an online algorithm that makes at most $\ldim$ mistakes on any realizable sequence.
    On the other hand, any randomized online algorithm has an expected number of mistakes on the worst sequence of at least $\nicefrac \ldim 2$.
    
    In the agnostic online setting, there is an algorithm with expected regret $O(\sqrt{T \cdot \ldim \cdot \log T})$, and any algorithm has expected regret at least $\Omega(\sqrt{T \cdot \ldim})$. 
\end{theorem}
Hence, it suffices to bound $\ldim$.
\begin{lemma}
    \label{lem:littlestone_upper_bound}
    $\ldim = O(n \log n)$.
\end{lemma}
\begin{proof}
    Assume for contradiction that there is a Littlestone-shattered tree $L$ of depth $d=2n\log(2n)+3$. Let $\mathcal{G}$ be the set of all hierarchical trees on the point set $V$, and notice that $|\mathcal{G}| \leq 2^{2n\log(2n)+1}$.
    We show that there is a root-to-leaf path $(v_0, \ldots, v_d)$ such that no hierarchical tree satisfies the constraints of the edges of this path.
    We choose the vertices $v_0,\dots,v_d$ in an iterative manner, and let $\mathcal{G}_i$ be the set of hierarchical trees that satisfy the constraints of the edges in the sub-path $(v_0,\dots,v_\tau)$.
    
    Initially, let $v_0$ be the tree's root, and $\mathcal{G}_0 = \mathcal{G}$.
    Given a we have chosen $v_0,\dots,v_\tau$ to be in the path, we choose $v_{i+1}$ as follows.
    Among two constraints $(v_\tau, u^{(1)})$ and $(v_\tau, u^{(2)})$ on the edges from $v_\tau$ to its children, at least one of them satisfies at most half of trees from $\mathcal{G}_\tau$, since these constraints are mutually contradictory. We select the corresponding child as $v_{\tau+1}$.
    
    After $d$ steps, it holds that $|\mathcal{G}_d| < 1$, meaning $\mathcal{G}_d$ is empty.
    Hence, there is no hierarchical tree that satisfies first $d$ constraints of this path, a contradiction to $L$ being Littlestone-shattered.
\end{proof}


In the remainder of the section, we show the following lemma.

\begin{lemma}
    \label{lem:littlestone_lower_bound}
    $\ldim = \Omega(n \log n)$.
\end{lemma}

\begin{proof}[Proof of Lemma~\ref{lem:littlestone_lower_bound}]
    Let $n'$ be the largest power of $k$ such that $n' \leq n$. We construct a Littlestone-shattered tree $L$ of depth $\frac{n'\log_k{n'}}{k} = \Omega(\frac{n\log_k{n}}{k^2})$, thus showing the lower bound on $\ldim$. Let $V' = \{x_1,\dots,x_{n'}\} \subseteq V$ be a set of $n'$ items of $V$.
    
\begin{definition}
    A tournament function $P$ is a function that given a set of points $X \subseteq V'$ whose size is a power of $k$, partitions $X$ into disjoint $k$-tuples.
\end{definition}

\begin{definition}
    For a set of data points $x_1,\dots,x_l$, a $(x_1,\dots,x_l)$-ladder is a tree shown in Figure~\ref{fig:ladder}.
    We say that $x_i$ has rank $i$ in the ladder.
\end{definition}
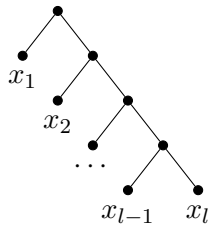
\begin{figure}[H]
    \centering
    \tikzset{
  solid node/.style={circle,draw,inner sep=1.2,fill=black},
}

\newcommand{\treeparams}{for tree={s sep=0.8cm,l-=1.5em,solid node}}
\bracketset{action character=@}

\begin{tikzext}
    \begin{forest}
        [ ,@\treeparams,
          [ ,label={below:$x_1$}]
          [
            [ ,label={below:$x_2$}]
            [
              [ ,label={below:$\cdots$}]
              [
                [ ,label={below:$x_{l-1}$} ]
                [ ,label={below:$x_{l}$} ]
              ]
            ]
          ]
        ];
    \end{forest}
\end{tikzext}
    \caption{$(x_1, \ldots, x_{l})$-ladder.}
    \label{fig:ladder}
\end{figure}

\begin{algorithm}[t!]
    \caption{\textsc{BuildLittleStoneTree}$(v, X_1, \ldots, X_\ell)$}
    \label{alg:build_littlestone_tree}
    \SetKwInOut{Input}{input}
    \SetKwInOut{Output}{output}
    \Input{Tree node $v$; partition of $V'$ into equal-sized sets $X_1^{(v)}, \ldots, X_\ell^{(v)}$, where $\ell$ is a power of $k$}
    \Output{A subtree of the current node in the Littlestone tree}
    \If{$|X_i^{(v)}| = 1$}{
        \Return single leaf node
    }
    Let $\{t_1, \ldots,t_{\nicefrac{n'}{k}}\} = P(X_1) \cup ... \cup P(X_\ell)$ \\
    Create a complete binary tree $\tilde L$ of depth $\nicefrac{n'}{k}$ as follows \\
    \For{$\tau=1, \ldots, \nicefrac{n'}{k}$}{
        Label every node on layer $\tau$ as $t_\tau$ \\
        Label outgoing edges from each node on layer $\tau$ as $D(t_\tau)$ and $U(t_\tau)$ \\
    }
    Let $\tilde L$ be the resulting tree \\
    \For{each leaf $u$ of $\tilde L$}{
        Initalize $X_1^{{(u)}}, \ldots, X_\ell^{{(u)}}$ as empty sets \\
        \For{each $x \in V'$} {
            Let $i \ge 0$ be unique index such that $t_\tau \in X_{i + 1}^{(v)}$ \\
            There exists a unique edge on the path from the root of $\tilde L$ to $u$ containing $x$ \\
            Let $r$ be the rank of $x$ in the label of this edge (recall that the edge label is a ladder) \\
            Let $j$ be $i \cdot k + r$ \\
            Assign $x$ to $X_j^{{(u)}}$
        }
        In $\tilde L$, replace $u$ with \textsc{BuildLittleStoneTree}$(X_1^{{(u)}}, \ldots, X_\ell^{{(u)}})$
    }
    \Return $\tilde L$
\end{algorithm}

Let $P$ be some arbitrary tournament function. Let $D$ be a function that maps any $k$-tuple $(x_1,\dots,x_k)$ to the $(x_1,\dots,x_k)$-ladder, where $x_1, \ldots, x_k$ are lexicographically ordered.
Similarly, let $U$ be a function that maps the set to the $(x_k,\dots,x_1)$-ladder (i.e. in the reversed lexicographic order). 

We describe the labeling of the Littlestone tree using a recursive process \textsc{BuildLittleStoneTree} (Algorithm~\ref{alg:build_littlestone_tree}), which receives as input subtree rooted $v$ whose edges and vertices are unlabeled, and a partition of $V'$ into disjoint sets $X_{1}^{(v)},\dots,X_{\ell}^{(v)}$ of equal size (which is a power of $k$).
\textsc{BuildLittleStoneTree} builds the Littlestone tree from top to bottom, by constructing next $\nicefrac{n'}{k}$ layers.
We initially run this procedure starting with $v$ being the root, $\ell=1$, and hence $X_1^{(v)} = V'$.
    
Let $v$ be a vertex on which a recursive call is made, and let $X_1^{(v)},\dots,X_\ell^{(v)}$ be the input partition.
If $|X_1^{(v)}| = \dots = |X_\ell^{(v)}| = 1$, then $u$ is the leaf of the Littlestone tree, and we halt.
Otherwise, we label the first $\nicefrac{n'}{k}$ layers as follows.
Let $t_\tau,\dots,t_{n'/k}$ be the $k$-tuples of $P(X_1^{(v)}) \cup ... \cup P(X_\ell^{(v)})$.
For $\tau=1,\dots,\nicefrac{n'}{k}$, we label all vertices of $\tau$'th layer in the subtree of $v$ with $t_\tau$, and its two edges with $D(t_\tau)$ and $U(t_\tau)$ respectively. 

Next, we need to call \textsc{BuildLittleStoneTree} from each node $u$ which is $\nicefrac{n'}{k}$ layers below $v$.
For that, we need to compute partition $X_{1}^{(u)}, \ldots, X_{1}^{(v)}$ and .
We need to assign each $x \in V'$ to some $X_i^{(u)}$, which we do as follows.
For each $x \in V'$ we compute the following quantities.
\begin{itemize}
    \item $\Ind^{(v)}(x)$ is the unique index $i \geq 0$ such that $x \in X_{i}^{(v)}$.
    \item On the path from $v$ to $u$, there exists a unique edge whose label (which is a ladder) contains $x$.
    Then $\Rank^{(v)}(x, u)$ is the rank of $x$ in that ladder.
\end{itemize} 
Then we assign $x$ to $X_j^{(u)}$, where $j = \Ind^{(v)}(x) \cdot k + \Rank^{(v)}(x, u)$.
In other word, we encode $j$ as a 2-digit number, where the highest digit equals $\Ind^{(v)}(x)$, and the lowest digit equals $\Rank^{(v)}(x, u)$.

\begin{lemma}
\label{lem:set_sizes}
Let $v$ be a node of depth $s \cdot \nicefrac{n'}{k}$ in $L$ for some integer $s$.
Then $|X_i^{(v)}| = \nicefrac{n'}{k^\tau}$ for any $1 \leq i \leq k^s$.
\end{lemma}
\begin{proof}
    We prove the claim by induction on $\tau$.
    For $s = 0$, the claim is trivial since for in root $r$ we have $X_1^{(r)} = V'$, and thus $|X_1^{(r)}| = n'$. 
    
    Assume by induction that the statement holds for $s < \log_k n'$
    Let $u$ be a node of depth $(s+1) \cdot \nicefrac{n'}{k}$, and let $v$ be its ancestor node in depth $s \cdot \nicefrac{n'}{k}$ in $L$.
    By the induction assumption, $X_i^{(v)} = \nicefrac{n'}{k^s}$ for any $1 \leq i \leq k^s$.
    First, note that elements from different $X_i^{(v)}$ are assigned to different $X_j^{(u)}$.
    
    It remains to consider elements from some fixed $X_i^{(v)}$.
    Each element $x \in X_i^{(v)}$ appears in exactly one ladder constraint in the path between $v$ and $u$.
    Aside from $x$, there are $k-1$ other items from $X_i^{(v)}$ in the same ladder, and each of them has a different rank.
    Therefore the number of elements from $X_i^{(v)}$ with $\Rank^{(v)}(x) = r$ is exactly $\frac{|X_j^{(v)}|}{k} = \frac{n'}{k^{s+1}}$ for any $1 \leq r \leq k$.
    Since elements from different $X_i^{(v)}$ are assigned to different sets $X_j^{(u)}$, the result follows.
\end{proof}

    
    
    

    \begin{lemma}
    \label{lem:online_order_preserve}
        Assume the label of some edge $(u,v)$ in $L$ is the $(x_1,\dots,x_k)$-ladder for some $x_1,\dots,x_k \in V'$. Then in any leaf descendant $w$ of $v$, and for any $1 \leq i < j \leq k$, it holds that $\Ind^{(w)}(x_{i}) < \Ind^{(w)}(x_{j})$.
    \end{lemma}
    \begin{proof}
        Let $v_1,\dots,v_\ell$ be the vertices on the path between between $v$ and $w$ (inclusive) on which \textsc{BuildLittleStoneTree} was called, i.e. those whose depth is a multiple of $\nicefrac{n'}{k}$).
        We prove by induction that for $1\leq s \leq \ell$ it holds that $\Ind^{(v_s)}(x_{i}) < \Ind^{(v_s)}(x_{j})$.
        
        Induction base.
        Let $p$ be the closest ancestor of $v$ (excluding $v$) on which a recursive call is made.
        We know that $\Ind^{(p)}(x_{i}) = \Ind^{(p)}(x_{j})$, since the ladder on edge $(u,v)$ contains both $x_i$ and $x_j$.
        Denoting $\Ind^* = \Ind^{(p)}(x_{i})$, we know that $x_{i}$ and $x_{j}$ are mapped to some $X_{k \cdot \Ind^* + r_1}$ and $X_{k \cdot \Ind^* + r_2}$ for some integers $1 \leq r_1 < r_2 \leq k$.
        Therefore, $\Ind^{(v_1)}(x_{i}) < \Ind^{(v_1)}(x_{j})$.
        
        Assume by induction that $\Ind^{(v_s)}(x_{i}) < \Ind^{(v_s)}(x_{j})$ for $s < \ell$.
        In the recursive call to $v_s$, every item $x$ in contained in a set $X_{k \cdot \Ind^{(v_s)}(x) + r}$ for some $1 \leq r \leq k$.
        Therefore, if $\Ind^{(v_s)}(x_{i}) < \Ind^{(v_s)}(x_{j})$ then $\Ind^{(v_{s+1})}(x_{i}) < \Ind^{(v_{s+1})}(x_{j})$. Since $v_\ell = w$, the claim follows.
        \end{proof}
    
    \begin{lemma}
    \label{lem:little_shattered}
        $L$ is Littlestone-shattered.
    \end{lemma}
    \begin{proof}
        Given a root-to-leaf path $v_1,\dots,v_\ell$, and assume that the leaf $v_l$ that was given partition $X_1^{(v_\ell)},\dots,X_{n'}^{(v_\ell)}$.
        Recall that by Lemma~\ref{lem:set_sizes}, any $X_i^{(v_\ell)}$ is of size $1$ for all $1 \leq i \leq n'$. Let $(x_1,\dots,x_{n'})$ be the points of $V'$ ordered such that $x_i$ is the unique item in $X_i^{(v_\ell)}$. 
        
        We show that a hierarchical tree defined as a $(x_1,\dots,x_{n'})$-ladder satisfies all constraints on the path $v_1,\dots,v_\ell$.
        Assume for contradiction that there are some $i < j$ such that $x_i$ and $x_j$ appear in a ladder constraint in one of the edges of the path $v_1,\dots,v_\ell$, such that $x_i$'s rank is larger than $x_j$.
        By Lemma~\ref{lem:online_order_preserve}, it holds that $j = \Ind^{(v_\ell)}(x_j)+1 < \Ind^{(v_\ell)}(x'_i)+1 = i$, leading to contradiction. 
    \end{proof}
    
        By construction, $L$ is of depth $\frac{n'\log_k{n'}}{k} = \Omega(\frac{n\log_k{n}}{k^2})$, and by Lemma~\ref{lem:little_shattered} $L$ is shattered, therefore $\ldim = \Omega(n\log{n})$.
        We conclude the proof of Lemma~\ref{lem:littlestone_lower_bound}.
    \end{proof}
    
    Since $\ldim = \Theta(n \log n)$, then Theorem~\ref{thm:online_main} follows immediately from Theorem 5.1, Theorem 5.2, and Theorem 5.3 of \cite{DanielySBS15}.
    
    \begin{theorem}[Formal version of Theorem~\ref{thm:online_main}]
    \label{thm:littlestone_main_formal}
    Let $k$ be any constant integer. In the realizable online setting, there exists an online algorithm that makes at most $O(n\log{n})$ mistakes on any realizable sequence. On the other hand, any randomized online algorithm has an expected number of mistakes on the worst sequence of at least $\Omega(n\log{n})$. In the agnostic online setting, there is an algorithm with expected regret $O(\sqrt{T \cdot n\log{n} \cdot \log (T)})$, and any algorithm has expected regret at least $\Omega(\sqrt{T \cdot n\log{n}})$ for a series of length $T$. 
\end{theorem}

    
    \fi

\end{document}